\documentclass[letterpaper]{article} % DO NOT CHANGE THIS
\usepackage{aaai25}  % DO NOT CHANGE THIS
\usepackage{times}  % DO NOT CHANGE THIS
\usepackage{helvet}  % DO NOT CHANGE THIS
\usepackage{courier}  % DO NOT CHANGE THIS
\usepackage[hyphens]{url}  % DO NOT CHANGE THIS
\usepackage{graphicx} % DO NOT CHANGE THIS
\usepackage{natbib}  % DO NOT CHANGE THIS AND DO NOT ADD ANY OPTIONS TO IT
\usepackage{caption} % DO NOT CHANGE THIS AND DO NOT ADD ANY OPTIONS TO IT
\usepackage{algorithm}
\usepackage{algorithmic}
\usepackage{amsmath,amsfonts,bm}

\usepackage{amsthm}
\usepackage{amssymb}
\usepackage{graphicx}
\usepackage{xcolor}
\usepackage{cite}
\usepackage{booktabs}
\usepackage{multicol}
\usepackage{multirow}
\usepackage{subcaption}
\usepackage{threeparttable}

\newtheorem{assumption}{Assumption}
\newtheorem{definition}{Definition}
\newtheorem{theorem}{Theorem}
\newtheorem{lemma}{Lemma}
\newtheorem{remark}{Remark}
\newtheorem{corollary}{Corollary}
\usepackage{newfloat}
\usepackage{listings}
\DeclareCaptionStyle{ruled}{labelfont=normalfont,labelsep=colon,strut=off} % DO NOT CHANGE THIS
\floatstyle{ruled}
\newfloat{listing}{tb}{lst}{}
\floatname{listing}{Listing}
\title{Memory-Reduced Meta-Learning with Guaranteed Convergence}
\author {
    Honglin Yang\textsuperscript{\rm 1,\rm 2},
    Ji Ma\textsuperscript{\rm 1,\rm 2}\equalcontrib,
    Xiao Yu\textsuperscript{\rm 2,\rm 3}
}
\affiliations {
    \textsuperscript{\rm 1}Department of Automation, Xiamen University, Xiamen, China\\
    \textsuperscript{\rm 2}Key Laboratory of Multimedia Trusted Perception and Efficient Computing, Ministry of Education of China, Xiamen, China\\
    \textsuperscript{\rm 3}Institute of Artificial Intelligence, Xiamen University, Xiamen, China\\
    yanghonglin@stu.xmu.edu.cn, maji08@xmu.edu.cn, xiaoyu@xmu.edu.cn
}
\usepackage{bibentry}
\begin{document}

\maketitle

\begin{abstract}
The optimization-based meta-learning approach is gaining increased traction because of its unique ability to quickly adapt to a new task using only small amounts of data. However, existing optimization-based meta-learning approaches, such as MAML, ANIL and their variants, generally employ backpropagation for upper-level gradient estimation, which requires using historical lower-level parameters/gradients and thus increases computational and memory overhead in each iteration. In this paper, we propose a meta-learning algorithm that can avoid using historical parameters/gradients and significantly reduce memory costs in each iteration compared to existing optimization-based meta-learning approaches. In addition to memory reduction, we prove that our proposed algorithm converges sublinearly with the iteration number of upper-level optimization, and the convergence error decays sublinearly with the batch size of sampled tasks. In the specific case in terms of deterministic meta-learning, we also prove that our proposed algorithm converges to an exact solution. Moreover, we quantify that the computational complexity of the algorithm is on the order of $\mathcal{O}(\epsilon^{-1})$, which matches existing convergence results on meta-learning even without using any historical parameters/gradients. Experimental results on meta-learning benchmarks confirm the efficacy of our proposed algorithm.
\end{abstract}

\section{Introduction}
Meta-learning has emerged as a promising paradigm in modern machine learning. Unlike traditional machine learning methods that require large amounts of data points for model training~\cite{traditionalML1, traditionalML2}, meta-learning leverages prior experience from previously learned tasks to quickly adapt to new tasks with limited data. This feature makes meta-learning particularly appealing for applications where data are scare or expensive to obtain~\cite{denevi2019online, infeasible}. To date, numerous meta-learning methods have been proposed, such as metric-based approaches~\cite{vinyals2016matching-metric1, snell2017prototypical-metric2}, optimization-based approaches~\cite{MAML, ANIL}, and model-based approaches~\cite{santoro2016meta-model1, li2019lgm-model2}. Recently, with the rapid progress of optimization theory in machine learning, optimization-based meta-learning approaches are gaining increased traction~\cite{Bertinetto_2018_bi_variant, imaml, lee2023shot, wu2023metazscil, pmlr-v235-chen24ap}.

Optimization-based approaches often formulate meta-learning as a bilevel programming problem, where the lower-level optimization is dedicated to adapting the model to a given task, and the upper-level optimization aims to train meta-parameters that can enable the model to generalize to new tasks~\cite{maclaurin2015gradient, zhang2024metadiff}.
% ~\cite{maclaurin2015gradient, lee2018gradient, zhang2024metadiff}
Over the past decades, numerous optimization-based meta-learning approaches have been proposed, with typical examples including Model-Agnostic Meta-Learning (MAML)  \cite{MAML} and its variants~\cite{finn2017one-variant, DEML-variant, Reptile, ANIL, BOIL_variant, chi2022metafscil, kirsch2022introducing}, and bilevel-optimization-based approaches~\cite{Franceschi_2018_bi_variant, Bertinetto_2018_bi_variant, zhou2019efficient_bi, Ji_2020_AID}. However, traditional optimization-based meta-learning approaches employ iterative differentiation to estimate upper-level gradients, i.e., hypergradient estimation, which requires the use of historical lower-level parameters/gradients, leading to increased computational and memory overhead, as evidenced by our experimental results in Figure~\ref{fig2} and Table~\ref{tab:metrics}. Although recent approaches, such as FOMAML~\cite{MAML} and others~\cite{Reptile, zhou2019efficient_bi, fallah2020convergence},
% ~\cite{reptile-variant, fallah2020convergence}
mitigate the memory costs in each iteration by avoiding the use of historical lower-level parameters/gradients, they suffer from decreased learning accuracy due to the lack of information from these lower-level parameters/gradients.

\subsection{Contributions}
\quad 1. We propose an optimization-based meta-learning algorithm that eliminates the need for using any historical parameters/gradients, and thus ensures nearly invariant memory costs over iterations. This is fundamentally different from traditional MAML-based and iterative-differentiation-based meta-learning approaches such as \citet{MAML, ANIL, BOIL_variant}, where memory costs grow to infinity when the iteration number of lower-level optimization tends to infinity. In fact, our experimental results in Figure~\ref{fig2} show that our algorithm achieves at least a $50\%$ reduction in memory costs compared to ANIL in~\citet{ANIL} and the iterative-differentiation-based approach in~\citet{Ji_2020_AID}.

2. In addition to reducing memory costs, we also establish the convergence rate and computational complexity of our algorithm for both stochastic and deterministic meta-learning, which is different from the existing bilevel programming literature, e.g.,~\citet{Franceschi_2018_bi_variant, shaban2019truncated, Ji_2020_AID}, that focus solely on deterministic cases. For stochastic meta-learning, we prove that our algorithm converges sublinearly with the number of iterations, and the convergence error decays sublinearly with the batch size of sampled tasks, which is consistent with the nonconvex nature of the stochastic meta-learning objective functions. For deterministic meta-learning, we prove that our algorithm can converge sublinearly to an exact solution.

3. We quantify that the computational complexity of our algorithm is on the order of $\mathcal{O}(\epsilon^{-1})$ for both stochastic and deterministic meta-learning, which matches the currently well-known results on optimization-based meta-learning~\cite{imaml, ji2020convergence}.

4. Furthermore, since our algorithm estimates the hypergradient using the product of the inverse of Hessian and vectors without computing the full Hessian or Jacobian matrices, we improve the computational complexity of the hypergradient estimation in previous meta-learning approaches~\cite{park2019meta-hessian2, hiller2022enforcing-hessian1} from $\mathcal{O}(q^2)$ (or $\mathcal{O}(pq)$) to $\mathcal{O}(\max\{p,q\})$ per iteration.

5. We conduct experimental evaluations using several meta-learning benchmark datasets, including the ``CIFAR-FS", ``FC100", ``miniImageNet", and ``tieredImageNet" datasets. The results confirm the efficiency of our algorithm in terms of both learning accuracy and memory reduction.

\subsection{Related Work}
The earliest optimization-based meta-learning approach is Model-Agnostic Meta-Learning (MAML), which is an initial parameter-transfer method where the objective is to learn a good ``optimal initial model parameters"~\cite{MAML}. Despite the widespread use of MAML and its variants such as DEML~\cite{DEML-variant}, ANIL~\cite{ANIL}, and BOIL~\cite{BOIL_variant} in image classification~\cite{chi2022metafscil, ullah2022meta_image-classification}, reinforcement learning~\cite{vuorio2019multimodal_rl, kirsch2022introducing}, and imitation learning~\cite{finn2017one-variant, li2021meta-il}, these approaches rely on iterative differentiation for hypergradient estimation, which requires using historical lower level parameters/gradients, leading to high computational and memory overhead. To solve this issue, ~\citet{MAML} proposed a simplified MAML algorithm that avoids leveraging historical lower-level parameters/gradients and uses only the last parameters in lower-level optimization for hypergradient updates, thereby reducing computational and memory overhead. However, the memory reduction achieved by this approach comes at the expense of degrading learning accuracy, which is undesirable in accuracy-sensitive meta-learning applications~\cite{zhou2019efficient_bi, peng2020comprehensiveoverviewsurveyrecent}.

Driven by the need of reducing computational and memory costs in optimization-based meta-learning approaches, plenty of bilevel-optimization-based algorithms have recently been proposed~\cite{Bertinetto_2018_bi_variant,Ji_2020_AID, pmlr-v235-chen24ap} (which aim to learn good embedding model parameters, rather than focusing on a good initialization as MAML-based approaches do). Initially, these algorithms addressed meta-learning problems by treating the lower-level optimality condition as constraints to the upper-level problem. However, this method is inapplicable to the scenario where the lower-level problem is large scale or the lower-level objective function has a complex structure. More recently, \citet{imaml, Ji_2020_AID} proposed an approximate-implicit-differentiation-based approach to solve meta-learning problems. However, this approach requires a common optimal solution for lower-level optimization, which is often difficult to satisfy in many meta-learning applications (note that in meta-learning, the lower-level optimization aims to adapt the model to different tasks, leading to heterogeneous optimal solutions for lower-level optimization problems). Moreover, the convergence analysis in~\citet{Ji_2020_AID} is based on deterministic gradients, and it is not clear if the deterministic convergence analysis methods can be extended to stochastic meta-learning, where the task distribution i-s unknown in practice.

\section{Preliminaries}\label{Section-2}
\paragraph{Notations} We use $\mathbb{R}$ and $\mathbb{N}^{+}$ to represent the sets of real numbers and positive integers, respectively. We abbreviate \textit{with respect to} by \textit{w.r.t.} We denote $\nabla F(\theta)\in\mathbb{R}^{p}$ as the gradient of $F(\theta)$. We use $\nabla_{\theta}g(\theta,\phi)$ and $\nabla_{\phi}g(\theta,\phi)$ to represent the gradients of $g$ \textit{w.r.t.} $\theta$ and $\phi$, respectively. We write $\nabla_{\theta\phi}^2g(\theta,\phi)\in\mathbb{R}^{p\times q}$ for the Jacobian matrix of $g$ and $\nabla_{\phi}^2g(\theta,\phi)\in\mathbb{R}^{q\times q}$ for the Hessian matrix of $g$ \textit{w.r.t.} $\phi$. For vectors $\phi_{1},\cdots,\phi_{n}$, we use $\boldsymbol{\phi}=\text{col}\{\phi_{1},\cdots,\phi_{n}\}$ to represent their stacked column vector. For a set $\mathcal{B}$, we denote the number of its elements by $|\mathcal{B}|$.

\subsection{Meta-Learning}
We consider a meta-learning problem with a series of tasks $\{\mathcal{T}_{i},~i\in{\mathbb{N}^{+}}\}$. All tasks are drawn from an unknown task distribution $\mathcal{P}(\mathcal{T})$. Each task $\mathcal{T}_{i}$ has a loss function $\mathcal{L}(\theta,\phi;\xi_i)$ over each data point $\xi_i$, where $\theta$ represents the meta-parameters of an embedding model shared by all tasks, and $\phi$ represents the task-specific parameters. The goal of meta-learning is to find common optimal parameters $\theta^*$ that benefit all tasks, and based on the parameters $\theta^*$, the model can quickly adapt its own parameters $\phi$ to a new task $\mathcal{T}_{i}$ using only a few data points and training iterations. Typically, the meta-learning problem can be formulated as the following bilevel programming problem:
\begin{equation}
\begin{aligned}
&\min_{\theta \in \mathbb{R}^{p}}\,F(\theta)=\mathbb{E}_{\mathcal{T}_{i}\sim\mathcal{P}(\mathcal{T})}[f_{i}(\theta,\phi_{i}^*(\theta))],\\
&\textrm{s.t.}\quad \phi_{i}^*(\theta)=\underset{\phi\in \mathbb{R}^{q}}{\text{argmin}}~g_{i}(\theta,\phi).\label{primal}
\end{aligned}
\end{equation}
In task $\mathcal{T}_{i}$, the upper-level objective function $f_{i}(\theta,\phi)$ is given by $f_{i}(\theta,\phi)\triangleq\frac{1}{|\mathcal{D}_{i,f}|}\sum_{j=1}^{|\mathcal{D}_{i,f}|}\mathcal{L}(\theta, \phi; \xi_{i,f}^{j})$, where data $\xi_{i,f}^{j},~j\in\{1,\cdots,|\mathcal{D}_{i,f}|\}$ is sampled from the validation dataset $\mathcal{D}_{i,f}$, and the objective function $g_{i}(\theta,\phi)$ is given by $g_{i}(\theta,\phi)\!\triangleq\!\frac{1}{|\mathcal{D}_{i,g}|}\sum_{j=1}^{|\mathcal{D}_{i,g}|}\mathcal{L}(\theta, \phi; \xi_{i,g}^{j})+R(\phi)$, where data $\xi_{i,g}^{j},~j\in\{1,\cdots,|\mathcal{D}_{i,g}|\}$ is sampled from training dataset $\mathcal{D}_{i,g}$ and $R(\phi)$ is a strongly-convex regularizer \textit{w.r.t.} $\phi$.

\begin{remark}
Compared with existing bilevel-programming frameworks which use common lower-level optimal parameters $\phi^*$ to enable upper-level optimization~(see, e.g.,~\citet{Ji_2020_AID}), our meta-learning framework in~\eqref{primal} allows different tasks $\mathcal{T}_{i}$ to have different lower-level optimal parameters $\phi_{i}^*$. This heterogeneity significantly complicates our convergence analysis.
\end{remark}
\begin{remark}
The meta-learning formulation in~\eqref{primal} involves bilevel-programming objectives. This is fundamentally different from conventional single-level machine learning, which focuses on learning optimal parameters $\phi^*$ for a specific task using a large number of data points~\cite{ soydaner2020comparison, chen2023quantized, chen2024locally}. However, in many applications, collecting data is costly and time consuming, and may even be infeasible due to physical system constraints~\cite{infeasible}.
% Therefore, meta-learning is typically regarded as a stronger framework than conventional single-level machine learning~\cite{Shu_Meng_Xu_2021_meta-vs-ml}.
Therefore, meta-learning is typically regarded as a more general framework than conventional single-level machine learning~\cite{Shu_Meng_Xu_2021_meta-vs-ml}.
In the special case where the upper-level objective function is absent, the meta-learning problem can be reduced to a single-level machine learning problem.
\end{remark}

The objective functions in~\eqref{primal} satisfy the following standard assumptions.
\begin{assumption}\label{assumption1}
$F(\theta)$ is nonconvex and $g_{i}(\theta,\phi)$ is $\mu$-strongly convex \textit{w.r.t.} $\phi$.
\end{assumption}
\begin{assumption}\label{assumption2}
We denote $f(\theta, \phi)\triangleq\mathbb{E}_{\mathcal{T}_{i}\sim\mathcal{P}(\mathcal{T})}[f_{i}(\theta,\phi)]$ and
$g(\theta,\phi)\triangleq\mathbb{E}_{\mathcal{T}_{i}\sim\mathcal{P}(\mathcal{T})}[g_{i}(\theta,\phi)]$. Then, the following statements are satisfied:

(i) Functions $f$, $\nabla f$,
$\nabla g$, $\nabla_{\theta\phi}^2g$, and $\nabla_{\phi}^2g$
are $l_{f,0}$, $l_{f}$, $l_{g}$, $l_{g,1}$, and $l_{g,2}$
Lipschitz continuous, respectively.

(ii) Gradients $\nabla f_i(\theta, \phi)$, $\nabla_{\phi}^2g_i(\theta, \phi)$ and $\nabla_{\theta\phi}^2g_i(\theta, \phi)$ are unbiased and have bounded variances $\sigma_{f1}^2$, $\sigma_{g1}^2$, and $\sigma_{g2}^2$, respectively.
\end{assumption}

Assumptions~\ref{assumption1} and~\ref{assumption2} are standard in meta-learning literature; see, for example~\citet{ji2020convergence}. Note that we do not impose that the lower-level function $g_{i}$ is Lipschitz \textit{w.r.t.} $\phi$, which is required in iMAML in~\citet{imaml}. Note that iMAML implicitly assumes the search space of parameters $\phi_{i}$ to be bounded such that function $g_{i}$ is bounded. Moreover, since each task $\mathcal{T}_{i},~i\in\{1,2,\cdots,|\mathcal{B}|\}$ is randomly sampled from the task distribution $\mathcal{P}(\mathcal{T})$, gradients $\nabla f_{i}(\theta,\phi)$ and $\nabla g_{i}(\theta,\phi)$ are stochastic, and thus we assume that they satisfy Assumption~\ref{assumption2}-(ii).

In practice, the task distribution $\mathcal{P}(\mathcal{T})$ is usually unknown, making it impossible to directly solve problem~\eqref{primal}. To circumvent this problem, a common approach is to reformulate~\eqref{primal} as the following empirical risk minimization problem:
\begin{equation}
\begin{aligned}
&\min_{\theta \in \mathbb{R}^{p}} F_{\mathcal{B}}(\theta)=f_{\mathcal{B}}\left(\theta, \boldsymbol{\phi}^*(\theta)\right)=\frac{1}{|\mathcal{B}|}\sum_{i=1}^{|\mathcal{B}|}f_{i}(\theta,\phi_{i}^*(\theta)),\\
&\textrm{s.t.}\quad \phi_{i}^*(\theta)=\underset{\phi \in \mathbb{R}^{q}}{\text{argmin}}~g_{i}(\theta, \phi),\label{primalB}
\end{aligned}
\end{equation}
where $\boldsymbol{\phi}^*(\theta)$ represents a stacked vector of task-specific optimal parameters $\phi_{i}^{*}$ and $|\mathcal{B}|$ denotes the batch size of sampled tasks.

\subsection{Main Challenges in Meta-Learning}
The main challenge in solving problem~\eqref{primalB} lies in computing the hypergradient $\nabla F_{\mathcal{B}}(\theta)$ of the upper-level function \textit{w.r.t.} $\theta$. From~\eqref{primalB}, the hypergradient computation requires knowledge of $\boldsymbol{\phi}^*(\theta)$. However, obtaining $\boldsymbol{\phi}^*(\theta)$ is often difficult in large-scale meta-learning, especially when the lower-level objective function $g_{i}(\theta,\phi)$ has a complex structure.

To solve this issue, a commonly used approach is to employ iterative differentiation for hypergradient estimation; see, for example,~\citet{domke2012generic, maclaurin2015gradient, Franceschi2017ForwardAR}. More specifically, the iterative-differentiation-based approach first executes a $K$-step gradient descent to solve the lower-level optimization problem in~\eqref{primalB} and obtain an approximate solution $\phi_{i}^{K}(\theta)$ for each task $\mathcal{T}_{i}$. Then, building on the solution $\phi_{i}^{K}(\theta)$, this approach estimates the hypergradient by using the following relation~\cite{Ji_2020_AID}:
\begin{flalign}
% &\nabla\widehat{F}_{\mathcal{B}}(\theta)=
&\frac{\partial f_{\mathcal{B}}(\theta, \boldsymbol{\phi}^{K}(\theta))}{\partial \theta}=
\frac{1}{|\mathcal{B}|}\sum_{i=1}^{|\mathcal{B}|}\frac{\partial f_{i}(\theta,\phi_{i}^{K}(\theta))}{\partial \theta}\nonumber\\
&=\frac{1}{|\mathcal{B}|}\sum_{i=1}^{|\mathcal{B}|}\left(\nabla_{\theta}f_{i}(\theta,\phi_{i}^{K}(\theta))-\lambda_{\phi}\sum_{k=0}^{K-1}\nabla_{\theta\phi}^{2}g_{i}(\theta,\phi_{i}^{k}(\theta))\times\right.\nonumber\\
&\left.\quad\prod_{q=k+1}^{K-1}\left(I\!-\!\lambda_{\phi}\nabla_{\phi}^2g_{i}(\theta,\phi_{i}^{q}(\theta))\right)\nabla_{\phi}f_{i}(\theta,\phi_{i}^{K}(\theta))\right).\label{chain}
\end{flalign}
It is clear that the hypergradient estimation in~\eqref{chain} relies on historical lower-level parameters $\phi_{i}^{k}(\theta),k=0,\cdots, K$. As the iteration number of lower-level optimization grows, this dependency significantly increases both computational and memory overhead, as evidenced by our experimental results in Figure~\ref{fig2} and Table~\ref{tab:metrics}. Moreover, obtaining a good approximation for $\phi_{i}^{*}(\theta)$ often requires numerous iterations in lower-level optimization, which inevitably extends the backpropagation chain and results in gradient problems such as vanishing or exploding hypergradient estimation~\cite{ji2020convergence, jamal2021lazy_itd-problem}.

Motivated by these observations, we aim to propose a meta-learning algorithm that can avoid using historical lower-level parameters/gradients for hypergradient estimation while still ensuring comparable convergence compared with existing optimization-based meta-learning approaches.

\section{Methodology}\label{Section-3}
In this section, we propose a memory-reduced algorithm for meta-learning with provable convergence to a solution $\theta^*$ to problem~\eqref{primalB}. Before presenting our algorithm, we first introduce our hypergradient-estimation approach.

\subsection{Hypergraident Estimation}
Inspired by the recent results on stochastic bilevel programming~\cite{ghadimi2018approximationmethodsbilevelprogramming, lorraine2020optimizing}, we estimate the hypergradient by using the following relation:
\begin{flalign}
\nabla F_{\mathcal{B}}(\theta)&= \nabla_{\theta}f_{\mathcal{B}}(\theta, \boldsymbol{\phi}^*(\theta))-
\frac{1}{|\mathcal{B}|}\sum_{i=1}^{|\mathcal{B}|}\left(
\nabla_{\theta\phi}^2g_{i}(\theta, \phi_{i}^*(\theta))\right.\nonumber\\
&\left.\quad \times [\nabla_{\phi}^{2}g_{i}(\theta, \phi_{i}^*(\theta))]^{-1}\nabla_{\phi}f_{i}(\theta, \phi_{i}^*(\theta))\right).\label{hypergradientestimate}
\end{flalign}

From~\eqref{hypergradientestimate}, it can be seen that obtaining $\nabla F_{\mathcal{B}}(\theta)$ requires computing the inverse of Hessian matrix $[\nabla_{\phi}^{2}g_{i}(\theta, \phi_{i}^*(\theta))]^{-1}$ and Jacobian matrix $\nabla_{\theta\phi}^{2}g_{i}(\theta, \phi_{i}^{*}(\theta))$. To avoid computing the full Hessian/Jacobian matrix, we aim to estimate the
Hessian-inverse-vector product:
\begin{equation}
v_{i}^*=[\nabla_{\phi}^{2}g_{i}(\theta, \phi_{i}^*(\theta))]^{-1}\nabla_{\phi}f_{i}(\theta, \phi_{i}^*(\theta)).\label{vi*}
\end{equation}
Based on~\eqref{hypergradientestimate}, the hypergradient can be rewritten as
\begin{equation}
\nabla F_{\mathcal{B}}(\theta)= \nabla_{\theta}f_{\mathcal{B}}(\theta, \boldsymbol{\phi}^*(\theta))-
\frac{1}{|\mathcal{B}|}\sum_{i=1}^{|\mathcal{B}|}
\nabla_{\theta\phi}^2g_{i}(\theta, \phi_{i}^*(\theta))v_{i}^*,\label{hyperB}
\end{equation}
where $\nabla_{\theta\phi}^2g_{i}(\theta, \phi_{i}^*(\theta))v_{i}^*$ will be referred to as the Jacobian-vector product. It follows from~\eqref{hyperB} that if the estimation of $\nabla_{\theta}f_{\mathcal{B}}(\theta, \boldsymbol{\phi}^*(\theta))$, $v_{i}^*$, and $\nabla_{\theta\phi}^2g_{i}(\theta, \phi_{i}^{*}(\theta))v_{i}^*$ is accurate enough, a good estimation of the hypergradient is obtained.

Note that estimating the
Hessian-inverse-vector product $v_{i}^*$ and further Jacobian-vector product $\nabla_{\theta\phi}^2g_{i}(\theta, \phi_{i}^*(\theta))v_{i}^*$ circumvents the requirement on estimating the full Hessian and Jacobian matrices, and hence, reduces the computational complexity from the order of $\mathcal{O}(q^2)$ (or $\mathcal{O}(pq)$) to the order of $\mathcal{O}(\max\{p,q\})$ per iteration. This is different from existing results on bilevel programming and meta-learning~\cite{park2019meta-hessian2, chen2022single,hiller2022enforcing-hessian1}, which estimate the hypergradient by computing full matrices directly, leading to heavy computational overhead.

\subsection{Memory-Reduced Meta-Learning Algorithm Design}
In this section, we first introduce an approach to estimate the Hessian-inverse-vector product $v_{i}^*$, which is necessary for estimating the hypergradient according to~\eqref{hyperB}. Using it as a subroutine, we will then propose our memory-reduced meta-learning algorithm.
\begin{algorithm}
\renewcommand{\thealgorithm}{1}
\floatname{algorithm}{Subroutine}
\caption{Estimating Hessian-inverse-vector product at the $t$-th outer loop iteration}
\label{subroutine1}
\begin{algorithmic}[1]
\STATE \textbf{Input:} Parameters  $\theta_t$ and $\phi_{i,t}^{K}$; integer $N$.
\STATE $v_{i,t}^{0}=v_{i,t-1}^{N}$ if $t>0$, and $v_{i,t}^{0}=\bm{0}_{q}$ otherwise.
\STATE $r_{i,t}^0=p_{i,t}^0=\nabla_{\phi}f_i(\theta_t,\phi_{i,t}^{K})-\nabla_{\phi}^2g_i(\theta_t, \phi_{i,t}^{K})v_{i,t}^0$.
%%%%%%%%%%%%%%%%%%%%%%%%%%%%
\FOR{$n=0,1...,N-1$}
\STATE {Get Hessian-vector product~$h_{i,t}^{n}=\nabla_{\phi}^{2}g_{i}(\theta_{t},\phi_{i,t}^{K})p_{i,t}^{n}$}.\\
\STATE $\eta_{i,t}^{n}=\frac{{r_{i,t}^{n}}^{T}r_{i,t}^{n}} {{p_{i,t}^{n}}^{T}h_{i,t}^{n}}$.
\STATE{$v_{i,t}^{n+1} = v_{i,t}^{n}+\eta_{i,t}^{n}p_{i,t}^{n}$}.
\STATE{$r_{i,t}^{n+1}=r_{i,t}^{n}-\eta_{i,t}^{n}h_{i,t}^{n}$
}.
\STATE{$\zeta_{i,t}^{n}=\frac{{(r_{i,t}^{n+1})}^{T}r_{i,t}^{n+1}}{{r_{i,t}^{n}}^{T}r_{i,t}^{n}}$}.
\STATE{$p_{i,t}^{n+1}=r_{i,t}^{n+1}+\zeta_{i,t}^{n}p_{i,t}^{n}$}.
\ENDFOR
\STATE{\textbf{Output: }$v_{i,t}^{N}$.}
\end{algorithmic}
\end{algorithm}

\begin{algorithm}[ht]
\renewcommand{\thealgorithm}{1}
\caption{Memory-reduced meta-learning algorithm}
\begin{algorithmic}[1]
\STATE{\textbf{Input:}} The batch size of sampled tasks $|\mathcal{B}|$; random initialization $\theta_{0}$ and $\phi_{i,0}$ for all $i\in\{1,\cdots,|\mathcal{B}|\}$; stepsizes $\lambda_{\phi}$ and $\lambda_{\theta}$; integers $T$ and $K$.
\FOR{$t=0,1,2,...,T-1$}
\STATE Sample a task batch $\mathcal{B}\sim \mathcal{P}(\mathcal{T})$.
\STATE $u_{0}=\bm{0}_{p}$.
\FOR{$i=1,\cdots,|\mathcal{B}|$}
\STATE {Set $\phi_{i,t}^{0}=\phi_{i,t-1}^{K}$ if $t>0$, and $\phi_{i,0}$ otherwise.}
\FOR{$k=0,1,...,K-1$}
\STATE $\phi_{i,t}^{k+1} = \phi_{i,t}^{k}-\lambda_{\phi} \nabla_{\phi}g_{i}(\theta_{t}, \phi_{i,t}^{k})$.
\ENDFOR
\STATE Run Subroutine~\ref{subroutine1} and obtain Hessian-inverse-vector product $v^{N}_{i,t}$.
\STATE Accumulate Jacobin-vector products $u_{i}= u_{i-1}+\nabla_{\theta\phi}^{2}g_{i}(\theta_{t},\phi_{i,t}^{K})v_{i,t}^{N}$.
\ENDFOR
\STATE Compute $\nabla\widehat{F}_{\mathcal{B}}(\theta_t)=\nabla_{\theta}f_{\mathcal{B}}(\theta_{t},\boldsymbol{\phi}_{t}^{K})-\frac{1}{|\mathcal{B}|}u_{|\mathcal{B}|}$.
\STATE $\theta_{t+1}=\theta_{t}-\lambda_{\theta} \nabla\widehat{F}_{\mathcal{B}}(\theta_t)$.
\ENDFOR
\STATE{\textbf{Output:}$\theta_{T}$}
\end{algorithmic}
\label{algorithm1}
\end{algorithm}

Approximating $v_{i}^*$ in~\eqref{vi*} equals to solving the equation $\nabla_{\phi}^{2}g_{i}(\theta, \phi_{i}^{*}(\theta))v_{i}=\nabla_{\phi}f_{i}(\theta, \phi_{i}^{*}(\theta))$, which is the optimality condition of the following optimization problem:
\begin{equation}
\min_{v\in \mathbb{R}^{q}}\varphi(v),\quad\varphi(v_{i})=\frac{1}{2}v_{i}^{T}H_{i}v_{i}-b_{i}^{T}v_{i},\label{Hessian-v-p}
\end{equation}
where $H_{i}$ and $b_{i}$ are given by $H_{i}=\nabla_{\phi}^{2}g_{i}(\theta, \phi_{i}^{*}(\theta))$ and $b_{i}=\nabla_{\phi}f_{i}(\theta, \phi_{i}^{*}(\theta))$, respectively.

We present Subroutine~\ref{subroutine1} to find the optimal solution $v_{i}^*$ to problem~\eqref{Hessian-v-p}.

Building on~Subroutine~\ref{subroutine1}, we can estimate the hypergradient in~\eqref{hyperB} by using the following equality:
\begin{equation}
\nabla\widehat{F}_{\mathcal{B}}(\theta_t)=\nabla_{\theta}f_{\mathcal{B}}(\theta_{t},\boldsymbol{\phi}_{t}^{K})-\frac{1}{|\mathcal{B}|}\sum_{i=1}^{|\mathcal{B}|}
\nabla_{\theta\phi}^2g_{i}(\theta, \phi_{i,t}^{K})v_{i,t}^{N}.\label{hyperC}
\end{equation}

With the hypergradient estimation~\eqref{hyperC}, we propose a memory-reduced algorithm for solving the meta-learning problem~\eqref{primal} in Algorithm~\ref{algorithm1}.

In Algorithm~\ref{algorithm1}, we use parameters $\phi_{i,t}^{K}$ at the last inner-loop iteration for hypergradient estimation. This is fundamentally different from existing iterative-differentiation-based meta-learning approaches, such as~\citet{MAML, antoniou2019trainmaml, ANIL, DEML-variant,BOIL_variant, yao2022metalearningfewertaskstask, baik2021meta}, which employ
historical lower-level parameters $\phi_{i}^{k}(\theta),k\in\{0,\cdots,K\}$ for hypergradient estimation; see~Eq.~\eqref{chain} for details.
Our Algorithm~\ref{algorithm1} significantly reduces memory costs in each outer-loop iteration $t$, which is evidenced by our experimental results in Figure~\ref{fig2}. Moreover, the avoidance of long-distance backpropagation in Algorithm~\ref{algorithm1} also reduces the risk of exploding or vanishing gradients, leading to better learning accuracy than existing iterative-differentiation-based meta-learning approaches in~\citet{MAML,ANIL,Ji_2020_AID}, as evidenced by our experimental results in Figure~\ref{fig1} and Table~\ref{tab:test accuracy}.

Inspired by the linear acceleration capability of conjugate gradient approaches in solving the quadratic programming problem~\cite{grazzi2020iteration}, we leverage the conjugate gradient in our Subroutine~\ref{subroutine1}. However, different from the existing conjugate gradient-based algorithm in~\citet{Ji_2020_AID} which relies on deterministic gradients, we employ stochastic gradients, which complicates our convergence analysis.
\begin{remark}
To deal with the bilevel programming objectives (where upper-level optimality relies on lower-level optimality), we employ nested-loop iterations in Algorithm~\ref{algorithm1}. Note that nested-loop iterations are commonly used in meta-learning algorithms, such as~\citet{ANIL,imaml}. The iteration numbers $K$ and $N$ in Algorithm~\ref{algorithm1} are fixed constants, which are independent of the outer-loop iteration number $T$. This is different from existing bilevel programming approaches in~\citet{chen2022single,chen2021closing} which have the inner-loop iteration number increasing with the outer-loop iteration, and hence have a heavier computational overhead. In addition, different from existing meta-learning algorithms in~\citet{hiller2022enforcing-hessian1, park2019meta-hessian2} which estimate the full Hessian matrix or Jacobian matrix, Algorithm~\ref{algorithm1} only estimates a vector of dimension $\max\{p,q\}$, and thus reduces computational complexity.
\end{remark}

\section{Convergence Analysis}\label{Section-4}
Algorithm~\ref{algorithm1} can ensure sublinear convergence with the number $T$ of outer-loop iteration, and the convergence error decreases sublinearly with the batch size of the sampled tasks. The results are summarized in Theorem~\ref{theorem1}, whose proof is given in Section~B of the supplementary material.

\begin{theorem}\label{theorem1}
Under Assumptions~\ref{assumption1} and~\ref{assumption2}, if the iteration numbers $K$ and $N$ satisfy $K \geq K_{0}$ and $N \geq N_{0}$ with detailed forms of $K_{0}$ and $N_{0}$ given in Section B.2 of the supplementary material, the iterates $\theta_{t}$ generated by Algorithm~\ref{algorithm1} satisfy
\begin{equation*}
\frac{1}{T+1}\!\sum_{t = 0}^{T}\mathbb{E}[{\left\|\nabla F
(\theta_{t})\right\|}^ {2}] \!\leq\! \mathcal{O}\left(\frac{1}{T}\right)+\mathcal{O}\left(\frac{1}{|\mathcal{B}|}\right).
\end{equation*}
\end{theorem}

Theorem~\ref{theorem1} proves that Algorithm~\ref{algorithm1} converges to a stable solution to problem~\eqref{primal} with the optimization error decreasing as the batch size of sampled tasks increases. We would like to point out that the bound $\mathcal{O}(\frac{1}{|\mathcal{B}|})$ in Theorem~\ref{theorem1}, caused by finite batch sizes of sampled tasks, inherently exists in all stochastic optimization approaches with finite samples~\cite{gower2019sgd_finite-sample}. Although the variance-reduction technique~\cite{reddi2016stochastic-variance_reduction, fang2018spider-variance_reduction} can be used to mitigate the influence of this term in single-level stochastic optimization, the extension of this approach to meta-learning/bilevel programming is hard to implement, since it is difficult to derive unbiased estimators of hypergradient, letting alone variance reduction ones; see~\citet{dagreou2022framework-variance_reduction} for details.

Moreover, to give a more intuitive description of the computational complexity, we define an $\epsilon$-solution to problem~\eqref{primal} as follows.
\begin{definition}\cite{lian2017can_decentralized}
For some positive integer $T$, if $\frac{1}{T}\sum_{t=0}^{T-1}\mathbb{E}[\|\nabla F(\theta_{t})\|^2]\leq \epsilon$ holds, then we say that the sequence $\{\theta_{t}\}_{t=0}^{T}$ can reach an $\epsilon$-solution to problem~\eqref{primal}.
\end{definition}

Building on Theorem~\ref{theorem1}, we have the following corollary.
\begin{corollary}\label{corollary1}
Under the conditions of Theorem~\ref{theorem1}, for any $\epsilon>0$, Algorithm~\ref{algorithm1} requires at most $\mathcal{O}((3+2K_{0})|\mathcal{B}|\epsilon^{-1})$ gradient evaluations on $\phi$ and $\mathcal{O}((1+K_{0}|\mathcal{B}|)\epsilon^{-1})$ gradient evaluations on $\theta$ to obtain an $\epsilon$-solution.
\end{corollary}

In Corollary~\ref{corollary1}, the inner-loop iteration number $K_{0}$ in Algorithm~\ref{algorithm1} is a fixed constant, which is different from the existing bilevel programming approaches in~\citet{chen2022single,chen2021closing} which have the inner-loop iteration number increasing with the outer-loop iteration, and hence have a higher computational complexity of the order of $\mathcal{O}(\epsilon^{-2})$. Moreover, the computational complexity of our Algorithm~\ref{algorithm1} matches the convergence results for MAML and ANIL in~\citet{ji2020convergence}, even when we use less memory in each outer iteration.

In Theorem~\ref{theorem1}, we consider a stochastic scenario in which tasks are drawn from an unknown task distribution $\mathcal{P}(\mathcal{T})$. Next, we consider a deterministic scenario in which we iterate over all sample elements in the task space. In this case, Assumption~\ref{assumption2} becomes a deterministic version as follows.

\begin{assumption}\label{assumption3}
For each task $\mathcal{T}_i$, functions $f_i$, $\nabla f_i$,
$\nabla g_i$, $\nabla_{\theta\phi}^2g_i$, and $\nabla_{\phi}^2g_i$
are $l_{f,0}$, $l_{f}$, $l_{g}$, $l_{g,1}$, and $l_{g,2}$
Lipschitz continuous, respectively.
\end{assumption}

\begin{figure*}[t]
\centering
\includegraphics[scale=0.40]{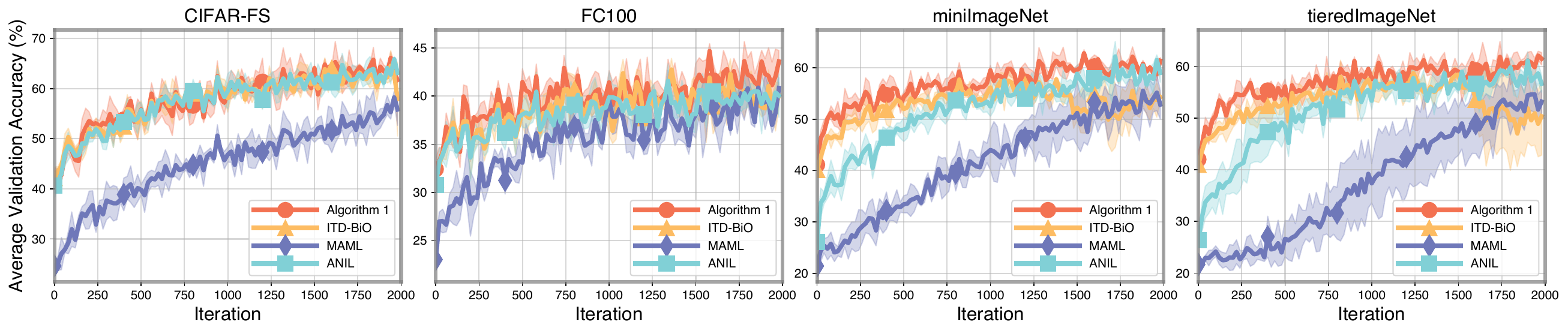}
\caption{Validation accuracies of Algorithm~\ref{algorithm1}, MAML, ANIL, and ITD-BiO in meta-learing stage using the ``CIFAR-FS", ``FC100", ``miniImageNet", ``tieredImageNet" datasets, respectively.}
\label{fig1}
\end{figure*}
\begin{table*}[t]
\centering
\begin{threeparttable}
\begin{tabular}{lcccccc}
\toprule
& & \textbf{CIFAR-FS} & \textbf{FC100} & & \textbf{miniImageNet} & \textbf{tieredImageNet} \\
\textbf{Algorithm} & \textbf{Backbone}\tnote{a} & \textbf{Accuracy(\%)} & \textbf{Accuracy(\%)} & \textbf{Backbone} & \textbf{Accuracy(\%)} & \textbf{Accuracy(\%)}\\
\midrule
MAML & 64-64-64-64 &  {54.85 $\pm$ 1.23}  & {47.50 $\pm$ 1.06}&32-32-32-32 &{43.33 $\pm$ 0.95}&{48.37 $\pm$ 1.17}\\
ANIL & 64-64-64-64 & {63.52 $\pm$ 1.31} & {47.70 $\pm$ 0.95} &32-32-32-32 &{57.80 $\pm$ 1.14}&{58.16 $\pm$ 0.94}\\
ITD-BiO & 64-64-64-64 & 63.69 $\pm$ 1.10 & {47.88 $\pm$ 0.74} &32-32-32-32 &{55.58 $\pm$ 1.31}&{60.78 $\pm$ 0.10}\\
\textbf{Algorithm 1 (ours)} & 64-64-64-64 & \textbf{64.24 $\pm$ 1.41} & \textbf{48.52 $\pm$ 1.13}&32-32-32-32 &\textbf{{59.58 $\pm$ 1.08}}&\textbf{{61.97 $\pm$ 0.86}}\\
\bottomrule
\end{tabular}
\begin{tablenotes}
\footnotesize
\item[a] The ``Backbone" represents the remaining layers in the model except for the last linear layer. The numbers in ``64-64-64-64'' and ``32-32-32-32'' represent the number of filters in each convolutional layer in the backbone.
\end{tablenotes}
\caption{Test accuracies by fine-tuning the model learned by Algorithm~\ref{algorithm1}, MAML, ANIL, and ITD-BiO using $20$-step gradient descent on the ``CIFAR-FS", ``FC100", ``miniImageNet", ``tieredImageNet" datasets, respectively.
}
\label{tab:test accuracy}
\end{threeparttable}
\end{table*}

\begin{theorem}\label{theorem2}
Under Assumptions~\ref{assumption1} and~\ref{assumption3}, if the iteration numbers $K$ and $N$ satisfy $K\geq \mathcal{O}(\kappa)$ and $N \geq\mathcal{O}\left(\sqrt{\kappa}\right)$ with $\kappa=\frac{l_{g}}{\mu}$, then the iterates $\theta_{t}$ generated by Algorithm~\ref{algorithm1} with deterministic gradients satisfy
\begin{equation*}
\frac{1}{T+1}\!\sum_{t = 0}^{T}{\left\|\nabla F
(\theta_{t})\right\|}^ {2} \leq\mathcal{O}\left(\frac{1}{T}\right).
\end{equation*}
\end{theorem}
Theorem~\ref{theorem2} demonstrates that when we consider deterministic meta-learning, Algorithm~\ref{algorithm1} converges to an exact solution to problem~\eqref{primal} with a sublinear convergence rate. To achieve an $\epsilon$-solution, Algorithm~\ref{algorithm1} with deterministic gradients requires at most $\mathcal{O}(\kappa\epsilon^{-1})$ gradient evaluations in both $\phi$ and $\theta$, which is consistent with the convergence results for the iterative-differentiation-based meta-learning algorithm in~\citet{Ji_2020_AID}.

\section{Experiments}\label{Section-5}
In this section, we evaluate the performance of our proposed Algorithm~\ref{algorithm1} by using a few-shot image classification problem on the ``CIFAR-FS" dataset~\cite{Bertinetto_2018_bi_variant}, the ``FC100" dataset~\cite{oreshkin2018tadam}, the ``miniImageNet" dataset~\cite{vinyals2016matching-metric1}, and the ``tieredImageNet" dataset~\cite{ren2018metalearningsemisupervisedfewshotclassification}, respectively. In all experiments, we compare our Algorithm~\ref{algorithm1} with other optimization-based meta-learning approaches, including the bilevel optimization-based algorithm with iterative differentiation (ITD-BiO) in~\citet{Ji_2020_AID}, MAML in~\citet{MAML}, and ANIL in~\citet{ANIL}. Due to space limitations, we leave the experimental setup in neural network training in Appendix~C.2.

\begin{figure*}[t]
\centering
\includegraphics[scale=0.4]{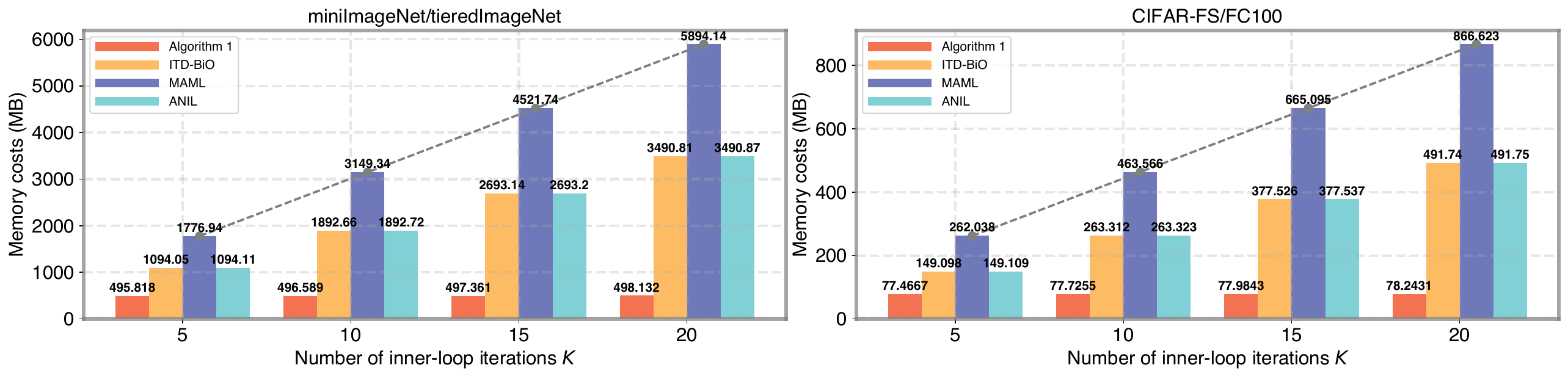}
\caption{Memory costs of Algorithm~\ref{algorithm1}, MAML, ANIL, and ITD-BiO under different numbers of inner-loop iterations $K$ using the ``CIFAR-FS", ``FC100", ``miniImageNet", ``tieredImageNet" datasets, respectively.}
\label{fig2}
\end{figure*}
\begin{table*}[t]
\centering
\begin{threeparttable}
\begin{tabular}{lcccc}
\toprule
& \multicolumn{1}{c}{\textbf{CIFAR-FS}} & \multicolumn{1}{c}{\textbf{FC100}} & \multicolumn{1}{c}{\textbf{miniImageNet}} & \multicolumn{1}{c}{\textbf{tieredImageNet}}\\
% \cmidrule(lr){2} \cmidrule(lr){3} \cmidrule(lr){4} \cmidrule(lr){5}
\textbf{Algorithm} & wallclock time(s)\tnote{a} & wallclock time(s) & wallclock time(s) & wallclock time(s)\\
\midrule
\textbf{MAML}& --\tnote{b} & 2638.45 & 10065.83 & 9453.75\\
\textbf{ANIL}& {1130.16} & 879.33 & 1639.85 & 1417.64\\
\textbf{ITD-BiO}& 1740.76 & 1607.74 & 3031.51 & 2059.43\\
\textbf{Algorithm 1 (ours)}& \textbf{832.40} & \textbf{621.67} & \textbf{739.62} & \textbf{693.46}\\
\bottomrule
\end{tabular}
\begin{tablenotes}
\footnotesize
\item[a] The wallclock time represents the time spent to achieve a certain validation accuracy of $0.90\times Acc_{\max}$, where $Acc_{\max}$ denotes the highest validation accuracy achieved among all comparison algorithms.
\item [b] Since MAML can never reach $0.90 \times Acc_{\max}$, we have not provided its wallclock time here.
\end{tablenotes}
\end{threeparttable}
\caption{
Comparison with MAML, ANIL, and ITD-BiO in terms of wallclock time using the ``CIFAR-FS", ``FC100", ``miniImageNet", ``tieredImageNet" datasets, respectively.
}
    \label{tab:metrics}
\end{table*}

\subsection{Evaluation on Meta-Learning Accuracy}
Following~\citet{MAML}, we consider a meta-learning problem with $|\mathcal{B}|$ tasks $\{\mathcal{T}_{i}, i=1,\cdots,|\mathcal{B}|\}$ in each iteration. Each task $\mathcal{T}_{i}$ has a loss function $\mathcal{L}(\theta,\phi_{i};\xi)$ over each data sample $\xi$, where $\theta$ represents the parameters of an embedding model shared by all tasks, and $\phi_{i}$ represents the task-specific parameters in a given task $\mathcal{T}_{i}$. The goal of our meta-learning problem~\eqref{primal} is to find the optimal parameters $\theta^*$ that benefit all tasks, and based on $\theta^*$, the model can quickly adapt its own parameters $\phi_{i}$ to any new task $\mathcal{T}_{i}$ using only a few data points and training iterations.

In this experiment, all comparative meta-learning algorithms are trained using a four-convolutional-layer convolutional neural network (CNN) architecture given in~\citet{ravi2016optimization}. The network is trained using a cross-entropy loss function in PyTorch. For our Algorithm~\ref{algorithm1} and ITD-BiO in~\citet{Ji_2020_AID}, $\phi_{i}$ corresponds to the parameters of the last linear layer of the CNN model and $\theta$ represents the parameters of the remaining layers. This setup ensures that the lower-level objective function $g_{i}(\theta,\phi)$ is strongly convex with respect to $\phi$, while the upper-level objective function $F(\theta)$ is generally nonconvex with respect to $\theta$.
For MAML and ANIL, we employ their default initialization-based learning paradigms given in~\citet{MAML} and~\citet{ANIL}, respectively.
All algorithms are executed in $|\mathcal{B}|=32$ batches of tasks over $2,000$ iterations with each task involving a training dataset $\mathcal{D}_{i}^{\text{tra}}$ and a validation dataset $\mathcal{D}_{i}^{\text{val}}$, both designed for $5$-way classification with $5$ shots for each class.
For the experiments conducted on the ``CIFAR-FS'' and ``FC100'' datasets, we set the step sizes to $\lambda_{\theta}=0.0001$ and $\lambda_{\phi}=0.1$; and for the experiments conducted on the ``miniImageNet'' and ``tieredImageNet'' datasets, we set $\lambda_{\theta}=0.001$ and $\lambda_{\phi}=0.05$.
The iteration numbers $K$ and $N$ are configured as $K=N=20$.
In our comparison, the near-optimal stepsizes and inner-loop iteration numbers are applied to ITD-BiO, MAML and ANIL, such that doubling them leads to divergent behavior.
Moreover, for ITD-BiO, MAML, and ANIL, we used Adam~\cite{kingma2017adammethodstochasticoptimization} for upper-level optimization and gradient descent for lower-level optimization, following the guidance from~\citet{Ji_2020_AID},~\citet{MAML}, and~\citet{ANIL}, respectively.

Note that evaluating the performance of a meta-learning algorithm involves two stages: learning-accuracy evaluation in the meta-learning stage and generalization-ability evaluation in the meta-test stage. The learning accuracy in the meta-learning stage is shown in Figure~\ref{fig1}. These results confirm the advantage of the proposed algorithm over existing meta-learning algorithms on various datasets. To show that Algorithm~\ref{algorithm1}'s generalization ability to a new task which is comprised of some images unseen in the meta-learning stage, we list the test accuracies by fine-tuning the model learned by Algorithm~\ref{algorithm1}, MAML, ANIL, and ITD-BiO using a $20$-step gradient descent. The results are summarized in Table~\ref{tab:test accuracy}, which shows that our proposed Algorithm~\ref{algorithm1} has better generalization ability than existing counterparts.

\subsection{Evaluation on Memory Costs}
In optimization-based meta-learning algorithms, a significant portion of memory costs come from storing historical lower-level parameters/gradients for hypergradient estimation (see Eq.\eqref{chain}). To show that Algorithm~\ref{algorithm1} can ensure nearly time-invariant memory usage over iterations, we compare the memory needed under different numbers of inner-loop iterations. The results in Figure~\ref{fig2} clearly show that the memory costs of Algorithm~\ref{algorithm1} remain consistently stable regardless of the number of inner-loop iterations. In contrast, the memory costs for MAML, ANIL, and ITD-BiO increase with the number of inner-loop iterations. Moreover, even when the number of inner-loop iterations is set to $5$, the memory consumption of our algorithm is more than $50\%$ lower compared to MAML, ANIL, and ITD-BiO.

\subsection{Evaluation on Wallclock Time}
To compare the convergence rate, i.e., computational complexity, of our Algorithm~\ref{algorithm1} with existing counterparts, we summarize their wallclock time to achieve a certain $0.90\times Acc_{\max}$ validation accuracy in Table~\ref{tab:metrics}, where $Acc_{\max}$ represents the highest validation accuracy achieved among all comparison algorithms. The experimental results validate that our algorithm is still faster than MAML, ANIL, and ITD-BiO in wallclock time even without using any historical lower-level parameters/gradients.

\section{Conclusion}\label{Section-6}
In this paper, we propose a meta-learning algorithm that can simultaneously reduce memory costs and ensure sublinear convergence. More specifically, our proposed approach avoids using any historical lower-level parameters/gradients, and hence, ensures nearly invariant memory costs over iterations. This is in sharp contrast to most existing iterative-differentiation-based algorithms for meta-learning, which use historical lower-level parameters/gradients to ensure learning accuracy, implying growing memory costs when the inner-loop iteration number increases. In addition, we systematically characterize the convergence performance of our algorithm for both stochastic
and deterministic meta-learning, and quantify the computational complexities for gradient evaluations on both upper-level and lower-level parameters. Experimental results on various benchmark datasets in few-shot meta-learning confirm the advantages of the proposed approach over existing counterparts.

\appendix
\section*{Technical Appendices}
\textbf{Outlines}
\begin{itemize}
\item Section~\ref{SectionS1}: Notations
\item Section~\ref{SectionS2}: Proof of Convergence for Algorithm~1
\begin{itemize}
\item~\ref{SectionS21} Technical Lemmas
\item~\ref{SectionS22} Proof of Theorem~1
\item~\ref{SectionS23} Proof of Corollary~1
\item~\ref{SectionS24} Proof of Theorem~2
\end{itemize}
\item Section~\ref{SectionS3}: Additional Experimental Results
\begin{itemize}
% \item ~\ref{SectionS31} Comparison on the sensitivity of learning rate
\item ~\ref{SectionS32} Benchmark Datasets Used in Experiments
\item ~\ref{SectionS33} Experimental Setup in Neural Network Training
\end{itemize}
\end{itemize}

\section{Notations}\label{SectionS1}
For the sake of clarity, the notations used in this paper are listed below.
\begin{itemize}
\item We define the true hypergradient at time $t$ as
\begin{flalign}
\nabla {F}(\theta_{t})&= \mathbb{E}_{\mathcal{T}_{i}\sim\mathcal{P}(\mathcal{T})}\left[\nabla_{\theta} f_{i}(\theta_t, {\phi}_{i,t}^{*})-\nabla_{\theta\phi}^2g_{i}(\theta_{t}, \phi_{i,t}^{*})\right.\nonumber\\
&\left.\quad\times \left(\nabla_{\phi}^{2}g_{i}(\theta_{t}, \phi_{i,t}^{*}\right)^{-1}\nabla_{\phi}f_{i}(\theta_{t}, \phi_{i,t}^{*})\right].
\end{flalign}
\item We use the first subscript (usually denoted as $i$) to represent the task index, and the second subscript (usually denoted as $t$) to represent the iteration index. For example, $x_{i,t}$ represents the variable $x$ of task $i$ at the $t$-th iteration. For the inner-loop iteration, $k$ represents the index of the inner-loop iteration. For example, $x_{i,t}^{k}$ denotes the variable $x$ of task $i$ at the $k$-th inner-loop iteration within the $t$-th outer-loop iteration.

\item We use $\mathbb{C}$ to represent the set of complex numbers. We use $\mathbb{R}$ and $\mathbb{R}^n$ to represent the set of real numbers, and the $n$-dimensional real space, respectively.
For a positive definite matrix $A\in\mathbb{R}^{n\times n}$ and a vector $x\in\mathbb{R}^n$, the matrix norm induced by
$A$ is defined as $\|x\|_{A}=\sqrt{x^TA^TAx}.$

\item Additional notations are introduced as follows: \\$A_{i,t}^{\mathcal{B}}\triangleq\nabla_{\phi}^{2}g_{i}(\theta_{t},\phi_{i,t}^{0})$, $b_{i,t}^{\mathcal{B}}\triangleq\nabla_{\phi}f_{i}(\theta_{t}, \phi_{i,t}^{0})$, $A^*_{i,t}\triangleq\nabla_{\phi}^{2}g_i(\theta_{t},\phi_{i,t}^{*})$, $b^*_{i,t}\triangleq\nabla_{\phi}f_i(\theta_{t},\phi_{i,t}^{*})$, $v^*_{i,t}\triangleq(A^{\mathcal{B}}_{i,t})^{-1}b^{\mathcal{B}}_{i,t}$, $\bar{v}^*_{i,t}\triangleq(A^*_{i,t})^{-1}b^*_{i,t}$, $\zeta^k_{i,t}=A^{\mathcal{B}}_{i,t}{e} ^k_{i,t}$, $\phi_{i,t}^*\triangleq\text{argmin}_{\phi\in\mathbb{R}^n}g_i(\theta_t,\phi)$
and $\psi=\max\{2\log(l_{f,0}),0\}+3\log(2)+\max\left\{\log\left(\frac{2\sigma^2_{f1}l^2_g}{\mu^2}+\frac{2\sigma^2_{g1}l^2_gl^2_{f,0}}{\mu^4}\right),0\right\}
+\max\{\log(36l_f^2),0\}. $

\end{itemize}
\section{Proof of Convergence for Algorithm 1}\label{SectionS2}
\subsection{Technical Lemmas}\label{SectionS21}

The following Lemma \ref{khfd} will establish the Lipschitz property of the gradient function $\nabla F$.

\begin{lemma}\label{khfd}
\cite{ghadimi2018approximationmethodsbilevelprogramming} Let $L_{\max}=\max\{l_f, l_g\}$. Under Assumptions~1 and~2, the Lipschitz continuous property of the upper-level objective function always holds for any $\theta$, $\theta'\in \mathbb{R}^p$:
\begin{equation}
\| \nabla F(\theta)-\nabla F(\theta') \| \le L_{F}\| \theta-\theta' \|,\nonumber
\end{equation}
where the positive constant $L_{F}$ is given by $L_{F}=L_{\max}+\frac{2L_{\max}^{2}+l_{g,1} l_{f,0}^{2}}{\mu}+\frac{L_{\max}^3+l_{f,0}L_{\max}(l_{g,1}+l_{g,2})}{\mu^2}+\frac{l_{g,2} L_{\max}^{2}l_{f,0}}{\mu^{3}}$.
\label{lemma-1}
\end{lemma}

Lemma~\ref{klda} quantifies the distance between $v^k_{i,t}$ and $v^*_{i,t}$, denoted as $\|v^k_{i,t} - v^*_{i,t}\|_{A^{\mathcal{B}}_{i,t}}$.
\begin{lemma}\label{klda}
%We denote $e_{i,t}^{k}\triangleq p_{i,t}^{k}-v_{i,t}$. If

There exists a decaying function  $\Gamma_k:\mathbb{N}\rightarrow\mathbb{R}$ satisfying
\begin{equation}
\Gamma_k=\Bigg\{\begin{aligned}
&a\gamma^k,~1\leq k<n,\\
&0,~~~~~k\geq n,
\end{aligned}
\end{equation}
with $a>1$ and $\gamma\in(0,1)$, such that the following inequality holds almost surely:
\begin{equation}\label{ahg}
\|v^k_{i,t} - v^*_{i,t}\|^2_{A^{\mathcal{B}}_{i,t}}\leq \Gamma_k\|v^0_{i,t} - v^*_{i,t}\|^2_{A^{\mathcal{B}}_{i,t}},
\end{equation}
for all $1\leq k\leq N$.
\end{lemma}
\begin{proof}
According to Section 9.1 in~\citet{shewchuk1994introduction}, we have the following relation:
\begin{flalign}
\|{e}_{i,t}^k\|^2_{A^{\mathcal{B}}_{i,t}}=\min_{\psi^k}\left\{\|{e}_{i,t}^{\psi^k}\|^2_{A^{\mathcal{B}}_{i,t}}\right\},
\end{flalign}
where $e_{i,t}^{k}\triangleq v_{i,t}^{k}-v_{i,t}$, $\psi^k$ and ${e}_{i,t}^{\psi^k}$ are given by $\psi^{k}\triangleq\{\psi_1,\cdots,\psi_k\}$ with $\psi_j\in\mathbb{C},~\forall j\in[1,k]$ and ${e}_{i,t}^{\psi^k}\triangleq\left(I+\sum_{j=1}^k\psi_j(A^{\mathcal{B}}_{i,t})^{2j}\right){e}_{i,t}^0$, respectively.

By using the Cayley-Hamilton theorem, we have that when $k\geq n$, there must exist a set $\{\psi_1,\cdots,\psi_k\}$ such that $I+\sum_{j=1}^k\psi_j(A^{\mathcal{B}}_{i,t})^{2j}=0.$ Furthermore, since $A^{\mathcal{B}}_{i,t}$ is a nonsingular matrix, we have
\begin{equation}
\|{e}_{i,t}^k\|^2_{A^{\mathcal{B}}_{i,t}}=0.
\end{equation}

In contrast, when $k<n$, we have the following relation:
\begin{equation}
\begin{aligned}
&\|{e}_{i,t}^k\|^2_{A^{\mathcal{B}}_{i,t}}\leq
\left\|\sum_{j=1}^k\psi_j\left((A^{\mathcal{B}}_{i,t})^{2j}-(A_{t})^{2j}\right){e}_{i,t}^0\right\|
_{(A^{\mathcal{B}}_{i,t})^{2j}}
\\
&\times(1+\frac{1}{2\eta})\!+\!(1+\frac{\eta}{2})\left\|\left(I+\sum_{j=1}^k\psi_j(A_{t})^{2j}\right){e}_{i,t}^0\right\|_{(A^{\mathcal{B}}_{i,t})^2},
\end{aligned}
\end{equation}
for any set $\psi^k=\{\psi_1,\cdots,\psi_k\}$ and any constant $\eta>0$.

Therefore, by choosing $\psi_1=\psi_2=\cdots=\psi_k=0$, the following inequality always holds almost surely:
$\|{e}_{i,t}^k\|^2_{A^{\mathcal{B}}_{i,t}}\leq\|{e}_{i,t}^0\|^2_{A^{\mathcal{B}}_{i,t}},$
which implies~\eqref{ahg} when the positive constant $a$ is greater than $1$ and $\gamma=\frac{1}{(a)^{1/n}}$.
\end{proof}

Lemma~\ref{lemma3} establishes an upper bound on $\mathbb{E}[\|\nabla\widehat{F}_{\mathcal{B}}(\theta_{t})-\nabla F(\theta_{t})\|^2]$.
\begin{lemma}\label{lemma3}
Under Assumptions~1 and~2, the estimated hypergradient generated by Algorithm~1 satisfies
\begin{equation}\label{jskag}
\begin{aligned}
\mathbb{E}[\|\nabla\widehat{F}_{\mathcal{B}}(\theta_{t})\!-\!\nabla F(\theta_{t})\|^2]\!\leq\!2\frac{\bar{\sigma}^2}{|\mathcal{B}|}\!+\!\frac{2}{|\mathcal{B}|}\sum_{i=1}^{|\mathcal{B}|}\!\mathbb{E}\!\left[\left\|A_{i,t}^{\mathcal{B}} e_{i,t}^N   \right\|^2\right]
\end{aligned}
\end{equation}
with $\bar{\sigma}^2:=2\sigma^2_{f1}+4\sigma^2_{g2}\frac{l^2_{f,0}}{\mu^2}+8l^2_{g,1}\left(\frac{\sigma^2_{f1}}{\mu^2}+\frac{\sigma^2_{g1}l^2_{f,0}}{\mu^4}\right).$
\end{lemma}
\begin{proof}
Define
\begin{equation}
\begin{aligned}
% &\nabla\widehat{F}_{\mathcal{B}}(\theta)=
&\nabla{F}_{\mathcal{B}}(\theta_{t})= \nabla_{\theta}f_{\mathcal{B}}(\theta_{t}, \boldsymbol{\phi}_{t}^{K})-
\frac{1}{|\mathcal{B}|}\sum_{i=1}^{|\mathcal{B}|}\left(
\left[ \nabla_{\theta\phi}^2g_{i}(\theta_{t}, \phi_{i,t}^{K}) \right]\right.\\
&\left.\quad \times \left[\nabla_{\phi}^{2}g_{i}(\theta_{t}, \phi_{i,t}^{K})\right]^{-1}\nabla_{\phi}f_{i}(\theta_{t}, \phi_{i,t}^{K})\right).
\end{aligned}
\end{equation}
Recalling the definition of $\nabla\widehat{F}_{\mathcal{B}}(\theta_{t})$ in~(8) and the dynamics $v_{i,t}$ generated by Subroutine~1, we have
\begin{equation}
\begin{aligned}
& \mathbb{E}\left[\left\|\nabla\widehat{F}_{\mathcal{B}}(\theta_{t})-\nabla F(\theta_{t})\right\|^2\right] \\
&\leq 2\mathbb{E}\left[\left\|\nabla{F}_{\mathcal{B}}(\theta_{t})-\nabla F(\theta_{t})\right\|^2\right] \\
&\quad +\frac{3}{2}\mathbb{E}\left[\left\|\nabla\widehat{F}_{\mathcal{B}}(\theta_{t})-\nabla F_{\mathcal{B}}(\theta_{t})\right\|^2\right] \\
% &\leq 2\left(1+\frac{l_{g}}{\mu}\right)^2\frac{l_{f,0}^2}{|\mathcal{B}|}
&\leq 2\frac{\bar{\sigma}^2}{|\mathcal{B}|} +2\mathbb{E}\left[\left\|\frac{1}{|\mathcal{B}|}\sum_{i=1}^{|\mathcal{B}|} \nabla_{\theta\phi}^2g_{i}(\theta_{t}, \phi_{t}^{K}) (v^N_{i,t}-v^*_{i,t}) \right\|^2\right] \\
&\leq 2\frac{\bar{\sigma}^2}{|\mathcal{B}|}+2\mathbb{E}\left[\left\|\frac{1}{|\mathcal{B}|}\sum_{i=1}^{|\mathcal{B}|}A_{i,t}^{\mathcal{B}} (v^N_{i,t}-v^*_{i,t}) \right\|^2\right] \\
&\leq 2\frac{\bar{\sigma}^2}{|\mathcal{B}|}+2\mathbb{E}\left[\left\|\frac{1}{|\mathcal{B}|}\sum_{i=1}^{|\mathcal{B}|}A_{i,t}^{\mathcal{B}} e_{i,t}^N \right\|^2\right], \label{adsa0}
\end{aligned}
\end{equation}
where in the derivation of the second inequality, we have used $\mathbb{E}\left[\left\|\nabla{F}_{\mathcal{B}}(\theta_{t})-\nabla F(\theta_{t})\right\|^2\right] \leq\bar{\sigma}^2$.

Consider the second term on the right side of the last equality in~\eqref{adsa0}, we have
\begin{equation}
\begin{aligned}
\mathbb{E}\left[\left\|\frac{1}{|\mathcal{B}|}\sum_{i=1}^{|\mathcal{B}|}A_{i,t}^{\mathcal{B}} e_{i,t}^N  \right\|^2\right]\leq \frac{1}{|\mathcal{B}|}\sum_{i=1}^{|\mathcal{B}|}\mathbb{E}\left[\left\|A_{i,t}^{\mathcal{B}} e_{i,t}^N   \right\|^2\right].
\label{dfw3rq10}
\end{aligned}
\end{equation}
The inequality~\eqref{dfw3rq10} holds by applying the following inequality:
\begin{equation}
\begin{aligned}
\left(\sum_{i=1}^na_i\right)^2\leq n\left(a_i^2\right),
\label{dfw3rq11}
\end{aligned}
\end{equation}
which holds for any $a_1,\cdots,a_n\in\mathbb{R}$.

Substituting \eqref{dfw3rq10} into \eqref{adsa0}, we have \eqref{jskag} holds.
\end{proof}

Lemma~\ref{lemma4} establishes convergence analysis for $\|\phi^{0}_{i,t}-\phi_{i,t}^*\|$.

\begin{lemma}\label{lemma4}
Denote that $\phi_{i,t}^*$ is the optimal solution for the function $g_i(\theta_t,\phi)$ \textit{w.r.t.} $\phi$. Then, $\mathbb{E}_{\mathcal{B}}\left[\left\|\phi^{0}_{i,t}-\phi_{i,t}^*\right\|^2\right]$ is upper bounded as follows:
\begin{flalign}
&\mathbb{E}\left[\|\phi^{0}_{i,t+1}-\phi_{i,t+1}^*\|^2\right]\leq 2(1-{\mu}\lambda_{\phi})^K \mathbb{E}\left[\|\phi^{0}_{i,t}-\phi_{i,t}^*\|^2 \right] \nonumber\\
&\quad +\!4 \frac{l^2_f \lambda^2_{\theta}}{\mu^2}\mathbb{E}\!\left[\!\|\nabla\widehat{F}_{\mathcal{B}}(\theta_{t})\|^2\!\right]\!+\!\frac{22\sigma^2_{g1}}{\mu^2}.\label{egse}
\end{flalign}
\end{lemma}
\begin{proof}
Since $\phi^{0}_{i,t+1}=\phi^{K}_{i,t},$ it follows from \eqref{dfw3rq11} that
\begin{equation}\label{adawr}
\begin{aligned}
\mathbb{E}\left[\|\phi^{0}_{i,t+1}-\phi_{i,t+1}^*\|^2\right]&\leq 2\mathbb{E}\left[\|\phi^{K}_{i,t}-\phi_{i,t}^*\|^2 \right]\\
&\quad+2 \mathbb{E}\left[\|\phi_{i,t+1}^*-\phi_{i,t}^*\|^2\right].
\end{aligned}
\end{equation}

(i) We first estimate an upper bound on
$\mathbb{E}\left[\|\phi^{K}_{i,t}-\phi_{i,t}^*\|^2 \right]$.

Since $\phi_{i,t}^*$ is the optimal solution for the function $g(\theta_t,\phi)$ \textit{w.r.t.} $\phi$, we have
\begin{equation}
\begin{aligned}
\nabla_{\phi}g_i(\theta_t,\phi_{i,t}^*)=0.\label{eq_20}
\end{aligned}
\end{equation}

Based on Step 8 in Algorithm 1, we have the dynamics of $\phi^k_{i,t}$ satisfying
\begin{equation}\label{hukeq}
\begin{aligned}
&\mathbb{E}\left[\|\phi^{k+1}_{i,t}-\phi_{i,t}^*\|^2\right]\\
&= \mathbb{E}\left[\left\|[\phi^{k}_{i,t}-\lambda_{\phi}\nabla_{\phi}g_i(\theta_t,\phi^k_{i,t})-\phi_{i,t}^*]\right\|^2\right]\\
&\leq\mathbb{E}\left[\left\|(\phi^{k}_{i,t}-\phi_{i,t}^*)\right\|^2\right]+\lambda^2_{\phi}\mathbb{E}\left[\|\nabla_{\phi} g_i(\theta_t,\phi^k_{i,t})\|^2\right]\\
&\quad-2\mathbb{E}\left[\lambda_{\phi}\left\langle \nabla_{\phi} g_i(\theta_t,\phi^k_{i,t}),\phi^{k}_{i,t}-\phi_{i,t}^*\right\rangle\right].
\end{aligned}
\end{equation}

Assumption~1 implies
\begin{equation}\label{khde}
\begin{aligned}
\left\langle \nabla g_i(\theta_t,\phi^k_{i,t}),\phi^{k}_{i,t}-\phi_{i,t}^*\right\rangle\geq\mu\|\phi^{k}_{i,t}-\phi_{i,t}^*\|^2.
\end{aligned}
\end{equation}

Moreover, by using Assumption~2, we have
\begin{equation}
\begin{aligned}\label{hge}
&\mathbb{E}\left[\|\nabla_{\phi}g_i(\theta_t,\phi^k_{i,t})\|^2\right]\\
&=\mathbb{E}\left[\|\nabla_{\phi}g_i(\theta_t,\phi^k_{i,t})-\nabla_{\phi}g_i(\theta_t,\phi_{i,t}^*)\|^2\right]\\
&\leq3\mathbb{E}\left[\|\nabla_{\phi}g(\theta_t,\phi^k_{i,t})-\nabla_{\phi}g(\theta_t,\phi_{i,t}^*)\|^2\right]+6\sigma^2_{g1}\\
&\leq 3l_g^2\mathbb{E}\left[\|\phi^k_{i,t}-\phi_{i,t}^*\|^2\right]+6\sigma^2_{g1}.
\end{aligned}
\end{equation}

Substituting~\eqref{khde} and~\eqref{hge} into~\eqref{hukeq} yields
\begin{equation}
\begin{aligned}
&~~~~~~\mathbb{E}\left[\|\phi^{k+1}_{i,t}-\phi_{i,t}^*\|^2\right]\\
&\leq (1-2\mu\lambda_{\phi}+3\lambda_{\phi}^2l^2_{g})\mathbb{E}\left[\|\phi^{k}_{i,t}-\phi_{i,t}^*\|^2\right]+6\lambda^2_{\phi}\sigma^2_{g1}.
\end{aligned}
\end{equation}

By choosing the stepsize satisfying $\lambda_{\phi}\leq\frac{3\mu}{l^2_g}$, we have
\begin{equation}
\begin{aligned}
\mathbb{E}\left[\|\phi^{k+1}_{i,t}-\phi_{i, t}^{*}\|^2\right]&\leq(1-\mu\lambda_{\phi})\mathbb{E}\left[\|\phi^{k}_{i,t}-\phi_{i,t}^*\|^2 \right]\\
&\quad+6\lambda^2_{\phi}\sigma^2_{g1},
\end{aligned}
\end{equation}
which implies
\begin{equation}
\begin{aligned}
\mathbb{E}\left[\|\phi^{K}_{i,t}-\phi_{i,t}^*\|^2 \right]\leq&(1-{\mu\lambda_{\phi}})^K\mathbb{E}\left[\|\phi^{0}_{i,t}-\phi_{i,t}^*\|^2 \right]\\
&+\frac{6\lambda^2_{\phi}\sigma^2_{g1}}{\mu}.
\end{aligned}\label{huye}
\end{equation}
%%%%%%%%%%%%%%%%%%%%%%%%%%%%%%%%%%%%%%%%%%%%%%%%%%%%%%%%%%%%%%%%%%
(ii) Then, we estimate an upper bound on $\mathbb{E}[\|\phi^*_{i,t+1}-\phi^*_{i,t}\|^2]$.

By using the following relation:
\begin{equation}
\begin{aligned}
\nabla_{\phi}g_i(\theta_{t+1},\phi^*_{i,t+1})=\nabla_{\phi}g_i(\theta_{t},\phi^*_{i,t})=0,
\end{aligned}
\end{equation}
we have
\begin{equation}
\begin{aligned}\label{khedaa}
&\nabla_{\phi}g_i(\theta_{t+1},\phi^*_{i,t+1})-\nabla_{\phi}g_i(\theta_{t+1},\phi^*_{i,t})\\&=\nabla_{\phi}g_i(\theta_{t},\phi^*_{i,t})
-\nabla_{\phi}g_i(\theta_{t+1},\phi^*_{i,t}).
\end{aligned}
\end{equation}

Combining the right-hand side of \eqref{khedaa} with
$\phi^*_{i,t+1}-\phi^*_{i,t}$, we obtain
\begin{equation}
\begin{aligned}
&\mathbb{E}\left[\left\langle\nabla_{\phi}g_i(\theta_{t+1},\phi^*_{i,t})\right.-\left.\nabla_{\phi}g_i(\theta_{t},\phi^*_{i,t}),\phi^*_{i,t+1}-\phi^*_{i,t} \right\rangle\right]\\
&=\mathbb{E}\left[\left\langle\nabla_{\phi}g_i(\theta_{t+1},\phi^*_{i,t})\right.\!-\!\left.\nabla_{\phi}g(\theta_{t+1},\phi^*_{i,t}),\phi^*_{i,t+1}-\phi^*_{i,t} \right\rangle\right]\\
&\quad+\mathbb{E}\left[\left\langle\nabla_{\phi}g(\theta_{t+1},\phi^*_{i,t})\right.-\left.\nabla_{\phi}g(\theta_{t},\phi^*_{i,t}),\phi^*_{i,t+1}-\phi^*_{i,t} \right\rangle\right]\\
&\quad+\mathbb{E}\left[\left\langle\nabla_{\phi}g(\theta_{t},\phi^*_{i,t})\right.-\left.\nabla_{\phi}g_i(\theta_{t},\phi^*_{i,t}),\phi^*_{i,t+1}-\phi^*_{i,t} \right\rangle\right].
\end{aligned}\label{bhgde}
\end{equation}

By using Assumption~2 and the Cauchy-Schwarz Inequality, we have
\begin{equation}
\begin{aligned}
&\mathbb{E}\left[\left\langle\nabla_{\phi}g(\theta_{t},\phi^*_{i,t})\right.-\left.\nabla_{\phi}g_i(\theta_{t},\phi^*_{i,t}),\phi^*_{i,t+1}-\phi^*_{i,t} \right\rangle\right]\\
&\leq \sqrt{\mathbb{E}\left[\left\|\nabla_{\phi}g(\theta_{t},\phi^*_{i,t})-\nabla_{\phi}g_i(\theta_{t},\phi^*_{i,t})\right\|^2\right]}\\
&~~~~\times\sqrt{\mathbb{E}\left[\left\|\phi^*_{i,t+1}-\phi^*_{i,t}\right\|^2\right]}\\
&\leq\sigma_{g1}\sqrt{\mathbb{E}\left[\left\|\phi^*_{i,t+1}-\phi^*_{i,t}\right\|^2\right]}.
\end{aligned}\label{arsw}
\end{equation}

Since $\nabla_{\phi}g$ is $l_g$-Lipschitz continuous \textit{w.r.t.} $\phi$, we have
\begin{equation}
\begin{aligned}
&\left\langle\nabla_{\phi}g(\theta_{t+1},\phi^*_{i,t})\right.-\left.\nabla_{\phi}g(\theta_{t},\phi^*_{i,t}),\phi^*_{i,t+1}-\phi^*_{i,t} \right\rangle\\
&\qquad\leq l_g\|\theta_{t+1}-\theta_{t}\|\|\phi^*_{i,t+1}-\phi^*_{i,t}\|.\label{ojhe}
\end{aligned}
\end{equation}

Since $\nabla_{\phi}g$ is $\mu$-strongly convex \textit{w.r.t.} $\phi$, we have
\begin{equation}
\begin{aligned}
&\left\langle\nabla_{\phi}g(\theta_{t+1},\phi^*_{i,t+1})\right.-\left.\nabla_{\phi}g(\theta_{t+1},\phi^*_{i,t}),\phi^*_{i,t+1}-\phi^*_{i,t} \right\rangle\\
&\quad\geq \mu\|\phi^*_{i,t+1}-\phi^*_{i,t}\|^2,
\end{aligned}\label{sawq}
\end{equation}
which implies
\begin{equation}
\begin{aligned}
&\mathbb{E}\left[\left\langle\nabla_{\phi}g(\theta_{t+1},\phi^*_{i,t+1})\right.-\left.\nabla_{\phi}g(\theta_{t+1},\phi^*_{i,t}),\phi^*_{i,t+1}-\phi^*_{i,t} \right\rangle\right]\\
&\geq \mu\mathbb{E}[\|\phi^*_{i,t+1}-\phi^*_{i,t}\|^2]-2\sigma_{g1}\sqrt{\mathbb{E}\left[\left\|\phi^*_{i,t+1}-\phi^*_{i,t}\right\|^2\right].
}
\end{aligned}\label{hgwq}
\end{equation}

Substituting~\eqref{arsw} and~\eqref{ojhe} into~\eqref{bhgde} and combining the resulting~\eqref{bhgde} with~\eqref{hgwq}, we obtain
\begin{equation}
\begin{aligned}
&\mu\mathbb{E}\left[\|\phi^*_{i,t+1}-\phi^*_{i,t}\|^2\right]\\
&\leq l_f\sqrt{\mathbb{E}\left[\left\|\phi^*_{i,t+1}-\phi^*_{i,t})\right\|^2\right]}\sqrt{\mathbb{E}\left[\left\|\theta_{t+1}-\theta_t\right\|^2\right]}\\
&~~~~+2\sigma_{g1}\sqrt{\mathbb{E}\left[\left\|\phi^*_{i,t+1}-\phi^*_{i,t}\right\|^2\right]},
\end{aligned}
\end{equation}
which implies
\begin{equation}\label{dawfa}
\begin{aligned}
\mathbb{E}\left[\|\phi^*_{i,t+1}-\phi^*_{i,t}\|^2\right]&\leq \frac{2l_f^2}{\mu^2_g}\mathbb{E}\left[\|\theta_{t+1}-\theta_t\|^2\right]+\frac{8\sigma^2_{g1}}{\mu^2}\\
&\leq \frac{2l_f^2\lambda^2_{\theta}}{\mu}\mathbb{E}\left[\|\nabla\widehat{F}_{\mathcal{B}}(\theta_{t})\|^2\right]+\frac{8\sigma^2_{g1}}{\mu^2}.
\end{aligned}
\end{equation}

Substituting \eqref{huye} and \eqref{dawfa} into \eqref{adawr}, we arrive at~\eqref{egse}.
\end{proof}
%%%%%%%%%%%%%%%%%%%%%%%%%%%%%%%%%%%%%%%%%%%%%%%%%%%%%%%%%%%%%%%%

%%%%%%%%%%%%%%%%%%%%%%%%%%%%%%%%%%%%%%%%%%%%%%%%%%%%%%%%%%%%%%%%%%%

Lemma~\ref{lemm5} establishes convergence analysis for variable $\zeta^k_{i,t}$.

\begin{lemma}\label{lemm5}
We denote $\zeta^k_{i,t}=A^{\mathcal{B}}_{i,t}{e} ^k_{i,t}$. Then, the following inequality always holds:
\begin{equation}\label{kjhed}
\begin{aligned}
\mathbb{E}\left[\|\zeta^N_{i,t}\|^2\right] &\leq2\Gamma_N\frac{l^2_g+\sigma^2_{g1}}{\mu^2} \mathbb{E}\left[\|\zeta^N_{i,t-1}\|^2\right]+\frac{C_1}{|\mathcal{B}|} \\
&\quad\times\mathbb{E}\left[\|\theta_t-\theta_{t-1}\|^2 \!+\! \|\phi^0_{i,t}-\phi^0_{i,t\!-\!1}\|^2\right]+\frac{C_2}{|\mathcal{B}|},
\end{aligned}
\end{equation}
where $C_{1}$ and $C_{2}$ are given by
\begin{equation}\label{kfwqkd}
\begin{aligned}
C_1&= 12\left(\frac{l^2_{f,0}}{\mu^2}+\sigma^2_{fg}\right){l^2_{g}}+12l^2_f,\\
C_2&=48\left(\frac{l^2_{f,0}}{\mu^2}+\sigma^2_{fg}\right)\sigma^2_{g1}+48\sigma^2_{f1},
\end{aligned}
\end{equation}
with $\sigma^2_{fg}:=2\frac{\sigma^2_{f1}}{\mu^2}+2\frac{\sigma^2_{g1}l^2_{f,0}}{\mu^4}$.

\end{lemma}

\begin{proof}
According to Step 2 in Subroutine 1, we have $v^0_{i,t}=v^N_{i,t-1}$, which further implies
\begin{equation}
\begin{aligned}
e^0_{i,t}=v_{i,t}^0-v^*_{i,t}&=v_{i,t-1}^N-v^*_{i,t-1}-(v^*_{i,t-1}-v^*_{i,t})\\
&=e^N_{i,t-1}-(v^*_{i,t-1}-v^*_{i,t}).
\end{aligned}
\end{equation}

By using the relation $\zeta^k_{i,t}=A^{\mathcal{B}}_{i,t}{e} ^k_{i,t}$, we have
\begin{equation}
\begin{aligned}\label{khd}
\zeta^0_{i,t}&=A^{\mathcal{B}}_{i,t}e^0_{i,t}=[A^{\mathcal{B}}_{i,t}e^N_{i,t-1}-A^{\mathcal{B}}_{i,t}(v^*_{i,t-1}-v^*_{i,t})]\\
&=A^{\mathcal{B}}_{i,t}(A^{\mathcal{B}}_{i,t-1})^{-1}\zeta^N_{i,t-1}-A^{\mathcal{B}}_{i,t}(v^*_{i,t-1}-v^*_{i,t}).
\end{aligned}
\end{equation}

By using Assumption~2, we obtain
\begin{equation}
\begin{aligned}
\mathbb{E}\left[\left\|A^{\mathcal{B}}_{i,t}(A^{\mathcal{B}}_{i,t-1})^{-1}\zeta^N_{i,t-1}\right\|^2\right]\leq\frac{l^2_g+\sigma^2_{g1}}{\mu^2}\mathbb{E}\left[\left\|\zeta^N_{i,t-1}\right\|^2\right].\label{xxx}
\end{aligned}
\end{equation}

Substituting~\eqref{xxx} into \eqref{khd} yields
\begin{equation}
\begin{aligned}\label{khgd}
\mathbb{E}\left[\left\|\zeta^0_{i,t}\right\|^2\right]&\leq2\frac{l^2_g+\sigma^2_{g1}}{\mu^2}\mathbb{E}\left[\left\|\zeta^N_{i,t-1}\right\|^2\right]\\
&\quad+2\mathbb{E}\left[\left\|A^{\mathcal{B}}_{i,t}(v^*_{i,t-1}-v^*_{i,t})\right\|^2\right].
\end{aligned}
\end{equation}

Since the relation $v^*_{i,t}=(A^{\mathcal{B}}_{i,t})^{-1}b^{\mathcal{B}}_{i,t}$ is always valid,
we have
\begin{equation}
\begin{aligned}
&\mathbb{E}\left[\left\|A^{\mathcal{B}}_{i,t}(v^*_{i,t-1}-v^*_{i,t})\right\|^2\right]\\
&=\mathbb{E}\left[\left\|\left[A^{\mathcal{B}}_{i,t}(A^{\mathcal{B}}_{i,t-1})^{-1}b^{\mathcal{B}}_{i,t-1}-b^{\mathcal{B}}_{i,t}\right]\right\|^2\right]\\
&=2\mathbb{E}\left[\left\|\left[(A^{\mathcal{B}}_{i,t}-A^{\mathcal{B}}_{i,t-1})(A^{\mathcal{B}}_{i,t-1})^{-1}b^{\mathcal{B}}_{i,t-1}\right]\right\|^2\right]\\
&~~~~+2\mathbb{E}\left[\left\|\left[b^{\mathcal{B}}_{i,t-1}-b^{\mathcal{B}}_{i,t}\right]\right\|^2\right].\label{eqad2}
\end{aligned}
\end{equation}

(a) To characterize~\eqref{eqad2}, we first estimate an upper bound on $\mathbb{E}\left[\left\|b^{\mathcal{B}}_{i,t-1}-b^{\mathcal{B}}_{i,t}\right\|^2\right]$.

According to the definition of
$b^{\mathcal{B}}_{i,t}$, we have
\begin{equation}
\begin{aligned}
&\mathbb{E}\left[\left\|b^{\mathcal{B}}_{i,t-1}-b^{\mathcal{B}}_{i,t}\right\|^2\right]\\
&=\mathbb{E}\left[\left\|\left[\nabla_{\phi}f_{i}(\theta_{t}, \phi_{i,t}^{0})-\nabla_{\phi}f_{i}(\theta_{t-1}, \phi_{i,t-1}^{0})\right]\right\|^2\right]\\
&\leq 3\mathbb{E}\left[\left\|\left[\nabla_{\phi}f_{i}(\theta_{t\!-\!1}, \phi_{i,t\!-\!1}^{0})\!-\!\nabla_{\phi}f(\theta_{t\!-\!1}, \phi_{i,t\!-\!1}^{0})\right]\right\|^2\right]\\
&\quad+3\mathbb{E}\left[\left\|\left[\nabla_{\phi}f(\theta_{t}, \phi_{i,t}^{0})-\nabla_{\phi}f(\theta_{t-1}, \phi_{i,t-1}^{0})\right]\right\|^2\right]\\
&\quad+3\mathbb{E}\left[\left\|\left[\nabla_{\phi}f_{i}(\theta_{t}, \phi_{i,t}^{0})-\nabla_{\phi}f(\theta_{t}, \phi_{i,t}^{0})\right]\right\|^2\right].\label{dfae}
\end{aligned}
\end{equation}

Assumption~2-(i) implies
\begin{equation}
\begin{aligned}
&\mathbb{E}\left[\left\|\left[\nabla_{\phi}f(\theta_{t}, \phi_{i,t}^{0})-\nabla_{\phi}f(\theta_{t-1}, \phi_{i,t-1}^{0})\right]\right\|^2\right]\\
&\leq{2l^2_{f}}\mathbb{E}\left[\left(\|\theta_t-\theta_{t-1}\|^2+|\phi^0_{i,t}-\phi^0_{i,t-1}\|^2 \right)\right].\label{fesdwe}
\end{aligned}
\end{equation}

Moreover, by using Assumption~2-(ii), we have
\begin{equation}
\begin{aligned}
\mathbb{E}\left[\left\|\left[\nabla_{\phi}f_{i}(\theta_{t}, \phi_{i,t}^{0})-\nabla_{\phi}f(\theta_{t}, \phi_{i,t}^{0})\right]\right\|^2\right]\leq\sigma^2_{f1}.\label{dfsdae}
\end{aligned}
\end{equation}

Substituting  \eqref{fesdwe} and \eqref{dfsdae} into \eqref{dfae} yields
\begin{flalign}
\mathbb{E}\left[\left\|\left[b^{\mathcal{B}}_{i,t-1}-b^{\mathcal{B}}_{i,t}\right]\right\|^2\right]\nonumber &\leq 6\sigma^2_{f1}+6{l^2_{f}}\left(\|\theta_t-\theta_{t-1}\|^2+\right.\\
&\left.\quad\|\phi^0_{i,t}-\phi^0_{i,t-1}\|^2\right).\label{dfasae}
\end{flalign}

(b) Then, we estimate an upper bound on the first term on the right-hand side of the second equation in~\eqref{eqad2} as follows:
\begin{equation}
\begin{aligned}
&\mathbb{E}\left[\left\|\left[(A^{\mathcal{B}}_{i,t}-A^{\mathcal{B}}_{i,t-1})(A^{\mathcal{B}}_{i,t-1})^{-1}b^{\mathcal{B}}_{i,t-1}\right]\right\|^2\right]\\
&=3\mathbb{E}\left[\left\|\left[(A^{\mathcal{B}}_{i,t}-A_{i,t})(A^{\mathcal{B}}_{i,t-1})^{-1}b^{\mathcal{B}}_{i,t-1}\right]\right\|^2\right]\\
&\quad+3\mathbb{E}\left[\left\|\left[(A_{i,t}-A_{i,t-1})(A^{\mathcal{B}}_{i,t-1})^{-1}b^{\mathcal{B}}_{i,t-1}\right]\right\|^2\right]\\
&\quad+3\mathbb{E}\left[\left\|\left[(A_{i,t-1}-A^{\mathcal{B}}_{i,t-1})(A^{\mathcal{B}}_{i,t-1})^{-1}b^{\mathcal{B}}_{i,t-1}\right]\right\|^2\right],\label{eqdasdaw}
\end{aligned}
\end{equation}
with
$A_{i,t}\triangleq\nabla_{\phi}^{2}g(\theta_{t},\phi_{i,t}^{0})$.

Next, we estimate the upper bounds on the three terms on the right-hand side of \eqref{eqdasdaw}.

(i)
By \eqref{dfw3rq11}, we have
\begin{equation}
\begin{aligned}\label{hjssaf-1}
&\mathbb{E}\left[\left\|\left[(A^{\mathcal{B}}_{i,t-1})^{-1}b^{\mathcal{B}}_{i,t-1}\right]\right\|^2\right]\leq2\mathbb{E}\left[\left\|\left[(A_{i,t-1})^{-1}b_{i,t-1}\right]\right\|^2\right]\\
&\quad+2\mathbb{E}\left[\left\|\left[(A^{\mathcal{B}}_{i,t-1})^{-1}b^{\mathcal{B}}_{i,t-1}-(A_{i,t-1})^{-1}b_{i,t-1}\right]\right\|^2\right].
\end{aligned}
\end{equation}

Assumptions~2-(i) implies
\begin{equation}
\begin{aligned}\label{hjssaf-2}
\mathbb{E}\left[\left\|\left[(A_{i,t-1})^{-1}b_{i,t-1}\right]\right\|^2\right]\leq\frac{l^2_{f,0}}{\mu^2}.
\end{aligned}
\end{equation}

Using Assumptions~2-(ii), we have
\begin{equation}
\begin{aligned}\label{hjsasf}
\mathbb{E}\left[\left\|\left[(A^{\mathcal{B}}_{i,t-1})^{-1}b^{\mathcal{B}}_{i,t-1}-(A_{i,t-1})^{-1}b_{i,t-1}\right]\right\|^2\right]\leq\sigma^2_{fg}.
\end{aligned}
\end{equation}

Substituting \eqref{hjssaf-2} and \eqref{hjsasf} into \eqref{hjssaf-1} yields
\begin{equation}
\begin{aligned}\label{hjsf}
\mathbb{E}\left[\left\|\left[(A^{\mathcal{B}}_{i,t-1})^{-1}b^{\mathcal{B}}_{i,t-1}\right]\right\|^2\right]\leq 2\frac{l^2_{f,0}}{\mu^2}+2\sigma^2_{fg}.
\end{aligned}
\end{equation}

Since the random variable $(A^{\mathcal{B}}_{i,t}-A_{i,t})$ is independent of
$(A^{\mathcal{B}}_{i,t-1})^{-1}b^{\mathcal{B}}_{i,t-1}$, we have
\begin{equation}\label{hjuw}
\begin{aligned}
&\mathbb{E}\left[\left\|\left[(A^{\mathcal{B}}_{i,t}-A_{i,t})(A^{\mathcal{B}}_{i,t-1})^{-1}b^{\mathcal{B}}_{i,t-1}\right]\right\|^2\right]\\
&=\mathbb{E}\left[\left\|\left[(A^{\mathcal{B}}_{i,t}-A_{i,t})\right\|^2\right]\mathbb{E}\left[\left\|(A^{\mathcal{B}}_{i,t-1})^{-1}b^{\mathcal{B}}_{i,t-1}\right]\right\|^2\right]\\
&\leq 2\left(\frac{l^2_{f,0}}{\mu^2}+\sigma^2_{fg}\right)\sigma^2_{g1}.
\end{aligned}
\end{equation}

(ii) Assumption~2 implies
\begin{equation}
\begin{aligned}\label{jidfr}
&\mathbb{E}\left[\left\|\left[(A_{i,t}-A_{i,t-1})(A^{\mathcal{B}}_{i,t-1})^{-1}b^{\mathcal{B}}_{i,t-1}\right]\right\|^2\right]\\
&\leq \left(\frac{l^2_{f,0}}{\mu^2}\!+\!\sigma^2_{fg}\right){l^2_{g}}\mathbb{E}\left[\left(\|\theta_t-\theta_{t-1}\|^2\!+\!|\phi^0_{i,t}-\phi^0_{i,t-1}\|^2 \right)\right].
\end{aligned}
\end{equation}

(iii)
By using Assumptions~1 and~2, we have
\begin{equation}\label{hidalojd}
\begin{aligned}
&\mathbb{E}\left[\left\|A_{i,t-1}(A^{\mathcal{B}}_{i,t-1})^{-1}b^{\mathcal{B}}_{i,t-1}-b^{\mathcal{B}}_{i,t-1}\right\|^2\right]\\
&\leq2\mathbb{E}\left[\left\|A_{i,t-1}(A^{\mathcal{B}}_{i,t-1})^{-1}b^{\mathcal{B}}_{i,t-1}\right\|^2\right]+2\mathbb{E}\left[\left\|b^{\mathcal{B}}_{i,t-1}\right\|^2\right]\\
&=2\left(\frac{l^2_{f,0}}{\mu^2}+\sigma^2_{fg}\right)\sigma^2_{g1}+2\sigma^2_{f1}.
\end{aligned}
\end{equation}

Substituting \eqref{hjuw}, \eqref{jidfr} and \eqref{hidalojd} into \eqref{eqdasdaw} yields
\begin{equation}
\begin{aligned}
&\mathbb{E}\left[\left\|\left[(A^{\mathcal{B}}_{i,t}-A^{\mathcal{B}}_{i,t-1})(A^{\mathcal{B}}_{i,t-1})^{-1}b^{\mathcal{B}}_{i,t-1}\right]\right\|^2\right]\\
&=3\left(\frac{l^2_{f,0}}{\mu^2}\!+\!\sigma^2_{fg}\right){l^2_{g}}\mathbb{E}\Big[\left(\|\theta_t\!-\!\theta_{t-1}\|^2\!+\!\|\phi^0_{i,t}\!-\!\phi^0_{i,t-1}\|^2\right)\Big]\\
&\quad+12\left(\frac{l^2_{f,0}}{\mu^2}+\sigma^2_{fg}\right)\sigma^2_{g1}+6\sigma^2_{f1}.\label{eqddtdaw}
\end{aligned}
\end{equation}

Moreover, by substituting \eqref{dfasae} and \eqref{eqddtdaw} into \eqref{eqad2}, we have
\begin{equation}
\begin{aligned}
&\mathbb{E}\left[\left\|A^{\mathcal{B}}_{i,t}(v^*_{i,t-1}-v^*_{i,t})\right\|^2\right]\\
&\leq\left[6\left(\frac{l^2_{f,0}}{\mu^2}+\sigma^2_{fg}\right){l^2_{g}}+12l^2_f\right]\mathbb{E}\bigg[\Big(\|\theta_t-\theta_{t-1}\|^2\\
&\quad+\|\phi^0_{i,t}-\phi^0_{i,t-1}\|^2 \Big)\bigg]+24\left(\frac{l^2_{f,0}}{\mu^2}+\sigma^2_{fg}\right)\sigma^2_{g1}+24\sigma^2_{f1}.\label{eqad21}
\end{aligned}
\end{equation}

According to Lemma~1, we have $\mathbb{E}\left[\|\zeta^N_{i,t}\|^2\right]\leq\Gamma_N \mathbb{E}\left[\|\zeta^0_{i,t}\|^2\right].$ There exists an integer $N$ not greater than $n$ such that $\Gamma_N$ is smaller than $\frac{1}{|\mathcal{B}|}$.
By using~\eqref{khgd} and~\eqref{eqad21}, we arrive at \eqref{kjhed}.
\end{proof}

\subsection{Proof of Theorem 1}\label{SectionS22}
\begin{proof}
Based on Step 14 in Algorithm~1, we have
\begin{equation}
\begin{aligned}
F(\theta_{t+1}) &\!\le\! F(\theta_{t})\!+\!\langle\nabla F(\theta_{t}),\theta_{t+1}-\theta_{t}\rangle\!+\!\frac{L_{F}}{2} {\| \theta_{t+1}-\theta_t \|}^2\\
&\le F(\theta_{t})-\lambda_{\theta}\langle\nabla F(\theta_{t}),\nabla\widehat{F}_{\mathcal{B}} (\theta_t)\rangle\\
&\quad+\frac{\lambda_{\theta}^2 L_{F}}{2} {\| \nabla\widehat{F}_{\mathcal{B}}(\theta_t)\|}^2,\label{eqasew5}
\end{aligned}
\end{equation}
where the constant $L_F$ is defined the same as it given in Lemma~\ref{lemma-1}.

Denote $\Xi_t=\langle\nabla F(\theta_{t}),\nabla\widehat{F}_{\mathcal{B}} (\theta_t)-\nabla F(\theta_{t})\rangle$. Then, we have
\begin{equation}
\begin{aligned}
\langle\nabla F(\theta_{t}),\nabla\widehat{F}_{\mathcal{B}} (\theta_t)\rangle= \Xi_t + \|\nabla F(\theta_{t})\|^2.\label{eqgs56}
\end{aligned}
\end{equation}

Substituting \eqref{eqgs56} into \eqref{eqasew5} yields
\begin{equation}
\begin{aligned}
\mathbb{E}[F(\theta_{t+1})] &\leq \frac{\lambda_{\theta}^2 L_{F}}{2}\mathbb{E}\left[{\|\nabla\widehat{F}_{\mathcal{B}}(\theta_{t})\!-\!\nabla F(\theta_{t})\|}^2\right]+\mathbb{E}[\Xi_t]\\
&\quad+\mathbb{E}[ F(\theta_{t})]-{\lambda_{\theta}}\mathbb{E}[\|\nabla F(\theta_{t})\|^2].\label{eq-57}
\end{aligned}
\end{equation}

By using the relation $\mathbb{E}[\Xi_t]=0$, we have
\begin{equation}
\begin{aligned}
{\lambda_{\theta}}\mathbb{E}[\|\nabla F(\theta_{t})\|^2] &\leq \lambda_{\theta}^2 L_{F}\frac{\bar{\sigma}^2}{|\mathcal{B}|}+\mathbb{E}[ F(\theta_{t})]-\mathbb{E}[F(\theta_{t+1})]\\
&\quad+\frac{2\lambda_{\theta}}{|\mathcal{B}|}\sum_{i=1}^{|\mathcal{B}|}\mathbb{E}\left[\left\|\zeta_{i,t}^N   \right\|^2\right].\label{eqsa57}
\end{aligned}
\end{equation}

Based on Step 14 in Algorithm~1, we have
\begin{equation}
\begin{aligned}
&\|\theta_t-\theta_{t-1}\|^2\leq\lambda^2_{\theta}\|\nabla\widehat{F}_{\mathcal{B}}(\theta_{t-1})^2\|\\
&\leq2\lambda^2_{\theta}\|\nabla{F}(\theta_{t-1})^2\|+2\lambda^2_{\theta}\|\nabla{F}(\theta_{t-1})-\nabla\widehat{F}_{\mathcal{B}}(\theta_{t-1})\|^2.\label{eakf10}
\end{aligned}
\end{equation}

By using~\eqref{jskag}, we have
\begin{equation}
\begin{aligned}
\mathbb{E}[\|\theta_t-\theta_{t-1}\|^2]&\leq2\lambda^2_{\theta}\mathbb{E}[\|\nabla{F}(\theta_{t-1})^2\|]\\
&\quad+4\lambda^2_{\theta}\frac{\bar{\sigma}^2}{|\mathcal{B}|}+\frac{2\lambda^2_{\theta}}{|\mathcal{B}|}\sum_{i=1}^{|\mathcal{B}|}\mathbb{E}\left[\left\|\zeta_{i,t}^N   \right\|^2\right].\label{eqsage0}
\end{aligned}
\end{equation}

Applying the inequalities~\eqref{dfw3rq11} and \eqref{dawfa}, we have
\begin{flalign}
\label{agye}
&\mathbb{E}[\|\phi^0_{i,t}-\phi^0_{i,t-1}\|^2]\leq 3\mathbb{E}[\|\phi^0_{i,t}-\phi^*_{i,t}\|^2]+\frac{24\sigma^2_{g1}}{\mu^2}\nonumber\\
&\quad+ \frac{6l_f^2\lambda^2_{\theta}}{\mu}\mathbb{E}\left[\|\nabla\widehat{F}_{\mathcal{B}}(\theta_{t})\|^2\right]+3\mathbb{E}[\|\phi^0_{i,t-1}-\phi^*_{i,t-1}\|^2]\nonumber\\
&\leq 3\mathbb{E}[\|\phi^0_{i,t}-\phi^*_{i,t}\|^2]+ 3\mathbb{E}[\|\phi^0_{i,t-1}\!-\!\phi^*_{i,t-1}\|^2]\nonumber\\
&\quad+\frac{12l^2_f\lambda^2_{\theta}}{\mu^2}\left(\mathbb{E}[\|\nabla\widehat{F}_{\mathcal{B}}(\theta_{t\!-\!1})\!-\!\nabla F(\theta_{t\!-\!1})\|^2]\right.\nonumber\\
&\quad\left.+\mathbb{E}[\|\nabla F(\theta_{t-1})\|^2]\right)+\frac{24\sigma^2_{g1}}{\mu^2}.
\end{flalign}

By using \eqref{jskag}, we have
\begin{flalign}
\label{afaye}
&\mathbb{E}[\|\phi^0_{i,t}-\phi^0_{i,t-1}\|^2]\leq 3\mathbb{E}[\|\phi^0_{i,t}-\phi^*_{i,t}\|^2]+\frac{24l^2_f\lambda^2_{\theta}}{\mu^2}\frac{\bar{\sigma}^2}{|\mathcal{B}|}\nonumber\\
&\quad+3\mathbb{E}[\|\phi^0_{i,t-1}-\phi^*_{i,t-1}\|^2]\nonumber+\frac{12l^2_f\lambda^2_{\theta}}{\mu^2}\mathbb{E}[\|\nabla F(\theta_{t-1})\|^2]\nonumber\\
&\quad+\frac{24l^2_f\lambda^2_{\theta}}{\mu^2|\mathcal{B}|}\sum_{i=1}^{|\mathcal{B}|}\mathbb{E}\left[\left\|\zeta^N_{i,t}   \right\|^2\right]+\frac{24\sigma^2_{g1}}{\mu^2}.
\end{flalign}

Substituting \eqref{eqsage0} and \eqref{afaye} into \eqref{kjhed} yields
\begin{equation}\label{kdajhed}
\begin{aligned}
&\mathbb{E}\left[\|\zeta^N_{i,t}\|^2\right]\leq2\Gamma_N\frac{l^2_g+\sigma^2_{g1}}{\mu^2} \mathbb{E}\left[\|\zeta^N_{i,t-1}\|^2\right]\\
&\quad+3\Gamma_N C_1\Big(\mathbb{E}[\|\phi^0_{i,t}-\phi^*_{i,t}\|^2]+\mathbb{E}[\|\phi^0_{i,t-1}-\phi^*_{i,t-1}\|^2]\Big)\\
&\quad+\frac{72C_1l^2_f\lambda^2_{\theta}}{\mu^2}\frac{\bar{\sigma}^2}{|\mathcal{B}|^2}+\frac{72C_1\sigma^2_{g1}}{\mu^2|\mathcal{B}|}+\frac{C_2}{|\mathcal{B}|}+12C_1\lambda^2_{\theta}\frac{\bar{\sigma}^2}{|\mathcal{B}|^2}\\
&\quad+\left(\frac{72C_1l^2_f\lambda^2_{\theta}}{\mu^2|\mathcal{B}|^2}+\frac{6C_1\lambda^2_{\theta}}{|\mathcal{B}|^2}\right)\sum_{i=1}^{|\mathcal{B}|}\mathbb{E}\left[\left\|\zeta^N_{i,t}   \right\|^2\right]\\
&\quad+\left(\frac{36C_1l^2_f\lambda^2_{\theta}}{\mu^2|\mathcal{B}|}+\frac{6C_1\lambda^2_{\theta}}{|\mathcal{B}|}\right)\mathbb{E}[\|\nabla F(\theta_{t-1})\|^2].
\end{aligned}
\end{equation}

Substituting \eqref{jskag} into \eqref{egse} yields
\begin{equation}\label{daadawe}
\begin{aligned}
&\mathbb{E}\left[\|\phi^{0}_{i,t+1}-\phi_{i,t+1}^*\|^2\right] \leq 2(1-{\mu}\lambda_{\phi})^K \mathbb{E}\left[\|\phi^{0}_{i,t}-\phi_{i,t}^*\|^2 \right] \\
&\quad +8 \frac{l^2_f \lambda^2_{\theta}}{\mu^2}\mathbb{E}\left[\|\nabla{F}(\theta_{t})\|^2\right]+16 \frac{l^2_f \lambda^2_{\theta}}{\mu^2}\frac{\bar{\sigma}^2}{|\mathcal{B}|}\\
&\quad +16 \frac{l^2_f \lambda^2_{\theta}}{\mu^2|\mathcal{B}|}\sum_{i=1}^{|\mathcal{B}|}\mathbb{E}\left[\left\|\zeta_{i,t}^N   \right\|^2\right]+\frac{16\sigma^2_{g1}}{\mu^2} .
\end{aligned}
\end{equation}

By choosing the number of inner-loop iteration $N$ satisfying
\begin{equation}
\begin{aligned}
N>&N_0=\Bigg[2 \log(\kappa)+\max\{2\log(l_{f,0}),0\}+3\log(2)\\
&+\max\left\{\log\left(\frac{2\sigma^2_{f1}l^2_g}{\mu^2}+\frac{2\sigma^2_{g1}l^2_gl^2_{f,0}}{\mu^4}\right),0\right\}\\
&+\max\{\log(36l_f^2),0\} \Bigg]/\log(\gamma^{-1}),
\end{aligned}
\end{equation}
we have
\begin{equation}
\begin{aligned}
2\Gamma_N\frac{l^2_g+\sigma^2_{g1}}{\mu^2}<\frac{1}{3}~\text{and}~1-3\Gamma_NC_1>\frac{2}{3}.
\end{aligned}
\end{equation}

Moreover, we choose the iteration number of Subroutine~1 $K$ satisfying
\begin{equation}
\begin{aligned}
K\geq K_0=\frac{\log(6)}{-\log(1-\mu\lambda_{\phi})}.
\end{aligned}
\end{equation}
Then, we have
\begin{equation}
\begin{aligned}
2(1-{\mu}\lambda_{\phi})^K<\frac{1}{3},
\end{aligned}
\end{equation}
which implies
\begin{equation}
\begin{aligned}
\frac{2(1-{\mu}\lambda_{\phi})^K+3\Gamma_NC_1}{1-3\Gamma_NC_1}<1.
\end{aligned}
\end{equation}

For notional simplicity, we define\\
\begin{equation}
\begin{aligned}\label{hgeaw}
\Xi_t&:=\mathbb{E}\left[\|\zeta^N_{t}\|^2\right]+\left(1-3\Gamma_NC_1\right)\\
&\qquad \times\frac{1}{|\mathcal{B}|}\sum_{i=1}^{|\mathcal{B}|}\mathbb{E}\left[\|\phi^0_{i,t}-\phi^*_{i,t}\|^2\right],\\
\rho&:=\max\left\{2\Gamma_N\frac{l^2_g+\sigma^2_{g1}}{\mu^2}+16 \frac{l^2_f \lambda^2_{\theta}}{\mu^2|\mathcal{B}|},\right.\\
&\left.\quad\quad\quad\quad\quad\quad\quad\quad\quad\frac{2(1-{\mu}\lambda_{\phi})^K+3\Gamma_NC_1}{1-3\Gamma_NC_1}\right\},\\
D_1&:=\frac{36C_1l^2_f\lambda^2_{\theta}}{\mu^2|\mathcal{B}|}+\frac{6C_1\lambda^2_{\theta}}{|\mathcal{B}|}+8 \frac{l^2_f \lambda^2_{\theta}}{\mu^2},\\
D_2&:=\frac{72C_1l^2_f\lambda^2_{\theta}}{\mu^2}\frac{\bar{\sigma}^2}{|\mathcal{B}|}+\frac{72C_1\sigma^2_{g1}}{\mu^2}+{C_2}\\
&\quad+12C_1\lambda^2_{\theta}\frac{\bar{\sigma}^2}{|\mathcal{B}|}+16 \frac{l^2_f \lambda^2_{\theta}}{\mu^2}\frac{\bar{\sigma}^2}{|\mathcal{B}|}.
\end{aligned}
\end{equation}

By combining \eqref{kdajhed} with \eqref{daadawe}, we have
\begin{equation}\label{adlkas-1}
\begin{aligned}
\Xi_{t+1}\leq \rho \Xi_t+D_1\mathbb{E}[\|\nabla F(\theta_{t})\|^2]+\frac{D_2}{|\mathcal{B}|}.
\end{aligned}
\end{equation}

By using \eqref{adlkas-1} and \eqref{eqsa57}, we have
\begin{equation}\label{aaswqas}
\begin{aligned}
&\sum_{t=0}^{T-1}\left[\left(\frac{\lambda_{\theta}}{2}-\lambda^2_{\theta}\right)\mathbb{E}[\|\nabla F(\theta_{t})\|^2]+\Xi_{t+1}\right]\\
&\leq\mathbb{E}[F(\theta_{0})]-\inf_{\theta}F(\theta)+\sum_{t=0}^T(2\lambda_{\theta}+2\lambda_{\theta}^2L_{F}+\rho)\Xi_t\\
&\quad+\sum_{t=0}^{T-1}D_1\mathbb{E}[\|\nabla F(\theta_{t})\|^2]+T\left[\frac{D_2}{|\mathcal{B}|}+2\lambda_{\theta}^2L_{F}\frac{\bar{\sigma}^2}{|\mathcal{B}|}\right].
\end{aligned}
\end{equation}

By choosing a sufficiently small parameter $\lambda_{\theta}$, the following inequalities hold simultaneously:
\begin{equation}\label{awqas}
\begin{aligned}
\frac{\lambda_{\theta}}{2}-\lambda^2_{\theta}+D_1\geq\frac{\lambda_{\theta}}{4},~\text{and}~
2\lambda_{\theta}+2\lambda_{\theta}^2L_{F}+\rho\leq 1.
\end{aligned}
\end{equation}
In this case, we have
\begin{equation}
\begin{aligned}
\sum_{t=0}^{T}\left[\frac{\lambda_{\theta}}{4}\mathbb{E}[\|\nabla F(\theta_{t})\|^2]\right]&\leq\mathbb{E}[F(\theta_{0})]-\inf_{\theta}F(\theta)+\Xi_0\\
&\quad+\!(T\!+\!1)\left[\frac{D_2}{|\mathcal{B}|}+2\lambda_{\theta}^2L_{F}\frac{\bar{\sigma}^2}{|\mathcal{B}|}\right].
\end{aligned}
\end{equation}

Multiplying both sides of the above expression by
$\frac{1}{T+1}$, we have
\begin{equation}\label{huifre}
\begin{aligned}
&\frac{1}{T+1}\sum_{t=0}^T\left[\frac{\lambda_{\theta}}{4}\mathbb{E}[\|\nabla F(\theta_{t})\|^2]\right]\\
&\leq\frac{1}{T+1}\left[\mathbb{E}[F(\theta_{0})]-\inf_{\theta}F(\theta)+\Xi_0\right]+\frac{D_2}{|\mathcal{B}|}\\
&\quad+(\lambda_{\theta}+2\lambda_{\theta}^2L_{F})\frac{\bar{\sigma}^2}{|\mathcal{B}|},
\end{aligned}
\end{equation}
which implies
\begin{equation}
\begin{aligned}
\frac{1}{T+1}\sum_{t=0}^{T}\left[\frac{\lambda_{\theta}}{4}\mathbb{E}[\|\nabla F(\theta_{t})\|^2]\right]\leq\mathcal{O}\left(\frac{1}{T}\right)+\mathcal{O}\left(\frac{1}{|\mathcal{B}|}\right).
\end{aligned}
\end{equation}
\end{proof}

\subsection{Proof of Corollary 1}\label{SectionS23}
\begin{proof}
According to Theorem 1, the convergence rate of Algorithm 1 is $\mathcal{O}\left(\frac{1}{T}\right)+\mathcal{O}\left(\frac{1}{|\mathcal{B}|}\right)$. Hence, to find an
$\epsilon$-optimal solution, the number of outer-loop iterations $T$ needs to satisfy $T=\mathcal{O}(\epsilon^{-1})$.
At each outer-loop iteration, Subroutine 1 requires
$3|\mathcal{B}|$ gradient evaluations in $\phi$; Step 8 in Algorithm 1  requires $K_0|\mathcal{B}|$ gradient evaluations in $\phi$; Step 11 in Algorithm 1 requires $K_0|\mathcal{B}|$ gradient evaluations in $\phi$ and  $K_0|\mathcal{B}|$ gradient evaluations in $\theta$; and Step 13 in Algorithm 1 requires one gradient evaluation  in $\theta$. Summing these results, we obtain our final conclusion: Algorithm~1 requires at most $(3+2K_{0})|\mathcal{B}|$ gradient evaluations in $\phi$ and $1+K_{0}|\mathcal{B}|$ gradient evaluations in $\theta$ at each outer-loop iteration.

Considering $T=\mathcal{O}(\epsilon^{-1})$, we can conclude that Corollary~1 holds.
\end{proof}

\subsection{Proof of Theorem 2}\label{SectionS24}
\begin{proof}
The proof of Theorem~2 closely resembles that of Theorem 1. Therefore, we will only highlight the key differences to complete the proof of Theorem~2.

First, according to \citet{shewchuk1994introduction}, we have
\begin{equation}
\begin{aligned}
\|{e}_{i,t}^N\|^2_{A^{\mathcal{B}}_{i,t}}\leq\left(\frac{\sqrt{\kappa}+1}{\sqrt{\kappa}-1}\right)^2N\|{e}_{i,t}^0\|^2_{A^{\mathcal{B}}_{i,t}}.
\end{aligned}
\end{equation}

Since the gradient function $\nabla F$ is deterministic, \eqref{jskag} can be rewritten as follows:
\begin{equation}
\begin{aligned}
\|\nabla\widehat{F}_{\mathcal{B}}(\theta_{t})-\nabla F(\theta_{t})\|^2\leq\frac{1}{|\mathcal{B}|}\sum_{i=1}^{|\mathcal{B}|}\left\|A_{i,t}^{\mathcal{B}} e_{i,t}^N   \right\|^2.
\end{aligned}
\end{equation}

Similarly, \eqref{egse} can be rewritten as follows:
\begin{equation}
\begin{aligned}
\|\phi^{0}_{i,t+1}-\phi_{i,t+1}^*\|^2&\leq 2(1-{\mu}\lambda_{\phi})^K \|\phi^{0}_{i,t}-\phi_{i,t}^*\|^2   \\
&\quad+2 \frac{l^2_f \lambda^2_{\theta}}{\mu^2}2\|\nabla\widehat{F}_{\mathcal{B}}(\theta_{t})\|^2.
\end{aligned}
\end{equation}
We can obtain the deterministic version of Equation \eqref{kjhed} as follows:
\begin{equation}\label{kjhded}
\begin{aligned}
\|\zeta^N_{i,t}\|^2 &\leq2\kappa^2\left(\frac{\sqrt{\kappa}+1}{\sqrt{\kappa}-1}\right)^{2N}\Big(\|\zeta^N_{i,t-1}\|^2\\
&\quad+\|\theta_t-\theta_{t-1}\|^2 + \|\phi^0_{i,t}-\phi^0_{i,t\!-\!1}\|^2\Big).
\end{aligned}
\end{equation}

We define
\begin{equation}
\begin{aligned}
\Xi_t&:=\|\zeta^N_{t}\|^2+\left(1-2v\right)\frac{1}{|\mathcal{B}|}\sum_{i=1}^{|\mathcal{B}|}\|\phi^0_{i,t}-\phi^*_{i,t}\|^2,\\
\rho&:=\max\left\{2v,~\frac{2(1-{\mu}\lambda_{\phi})^K+2v}{1-2v}\right\},\\
D_1&:=4v\lambda^2_{\theta},
\end{aligned}
\end{equation}
where $v=\kappa^2\left(\frac{\sqrt{\kappa}+1}{\sqrt{\kappa}-1}\right)^{2N}$.
Then, we can prove that when $K\geq \mathcal{O}(\kappa)$ and $N \geq\mathcal{O}\left(\sqrt{\kappa}\right)$, we have
\begin{equation}\label{adlkas}
\begin{aligned}
\Xi_{t+1}\leq \rho \Xi_t+D_1\|\nabla F(\theta_{t})\|^2.
\end{aligned}
\end{equation}
By using an argument similar to the derivation of \eqref{huifre}, we have
\begin{equation}\label{huifre-2}
\begin{aligned}
\frac{1}{T+1}\sum_{t=0}^T\!\left[\frac{\lambda_{\theta}}{2}\|\!\nabla \!F(\theta_{t})\!\|^2\right]\leq\frac{1}{T}\!\left[F(\theta_{0})-\inf_{\theta}F(\theta)+\Xi_0\right]\!,
\end{aligned}
\end{equation}
which implies
\begin{equation}
\begin{aligned}
\frac{1}{T+1}\sum_{t=0}^{T}\left[\|\nabla F(\theta_{t})\|^2\right]=\mathcal{O}\left(\frac{1}{T+1}\right).
\end{aligned}
\end{equation}
\end{proof}

\section{Additional Experimental Details}\label{SectionS3}

\subsection{Benchmark Datasets Used in Experiments}\label{SectionS32}
\noindent\textbf{CIFAR-FS}~\cite{Bertinetto_2018_bi_variant} is a widely utilized few-shot learning benchmark derived from the CIFAR-100 dataset~\cite{krizhevsky2009learning}. It consists of 100 classes grouped into 20 superclasses, with a distribution of 64, 16, and 20 classes for training, validation, and testing respectively. Each class includes 600 images with a resolution of $32\times32$ pixels.\vspace{0.1cm}

\noindent\textbf{FC100}~\cite{oreshkin2018tadam} is another few-shot learning benchmark newly derived from CIFAR-100. While it maintains a structure similar to the CIFAR-FS dataset, it features a distinct class distribution.
It also consists of 20 superclasses, divided into 12, 4, and 4 categories for the training, validation, and testing splits. This arrangement results in 60, 20, and 20 general classes across the respective splits, maintaining the same image count and resolution per class as CIFAR-FS.\vspace{0.1cm}

\noindent\textbf{miniImageNet}~\cite{vinyals2016matching-metric1} stands as a significant benchmark in the field of few-shot learning, designed to address the specific challenges of image classification. It is derived from the extensive ImageNet database but focuses on a subset of 100 classes. Each class is comprised of 600 images sized $84\times84$, resulting in a total of 60,000 images for the dataset.\vspace{0.1cm}

\noindent\textbf{tieredImageNet}~\cite{ren2018metalearningsemisupervisedfewshotclassification} is a more extensive subset of ILSVRC-12. It consists of 608 classes grouped into 34 superclasses that are partitioned into 20, 6, and 8 disjoint sets for training, validation, and testing respectively. This hierarchical organization not only makes it more challenging to understand the relations between classes but also simulates more complex classification tasks.

\subsection{Experimental Setup in Neural Network Training}\label{SectionS33}
\noindent\textbf{Network architecture.}\quad In the image classification experiments, we utilized a consistent network architecture for all comparison algorithms given by~\citet{arnold2020learn2learnlibrarymetalearningresearch}. Specifically, the backbone consists of four convolutional blocks, each of which includes $3\times3$ convolutions and filters, followed by batch normalization, ReLU activation, and $2\times2$ max-pooling in order. Following the setup in~\citet{arnold2020learn2learnlibrarymetalearningresearch}, the number of filters depends on the datasets. For ``CIFAR-FS'' and ``FC100'' datasets, we set it as $64$. For ``miniImageNet'' and ``tieredImageNet'' datasets, we set it as $32$.

\vspace{0.1cm}
\noindent\textbf{Hyperparameter selections.}\quad We briefly covered the setup of hyperparameters for each algorithm in Section~5. The detailed settings and selections are explained as follows. In the experiments conducted on ``CIFAR-FS'' and ``FC100'' datasets, for MAML, we set the inner-loop iteration number $K$ as 3 with a stepsize of $\lambda_{\phi}=0.5$, and the outer-loop stepsize $\lambda_{\theta}=0.001$; for ANIL, we set $K=10$ with a stepsize of $\lambda_{\phi}=0.1$, and the outer-loop stepsize $\lambda_{\theta}=0.0001$.
In the experiments conducted on ``miniImageNet'' and ``tieredImageNet'' datasets, for MAML, we set the inner-loop iteration number $K$ as 3 with a stepsize of $\lambda_{\phi}=0.5$, and the outer-loop stepsize $\lambda_{\theta}=0.005$; for ANIL, we set $K=5$ with a stepsize of $\lambda_{\phi}=0.1$, and the outer-loop stepsize $\lambda_{\theta}=0.001$.
In all experiments, for ITD-BiO and our algorithm, we conducted a grid search to choose $K$ from $\{5, 10, 15, 20\}$, $\lambda_{\phi}$ from $\{0.01, 0.05, 0.1, 0.5\}$, and $\lambda_\theta$ from $\{0.0001, 0.001, 0.01, 0.1\}$. Specifically for our algorithm, we chose the iteration number $N$ of Subroutine~1 from $\{5, 10, 15, 20\}$.

\section*{Acknowledgments}
This work was supported in part by the National Science and Technology Major Project under Grant 2021ZD0112600 and in part by the National Natural Science Foundation of China under Grants 62103343.

\bibliographystyle{aaai25}
\bibliography{aaai25}

\end{document}